\documentclass[preprint,12pt,times]{elsarticle}
\biboptions{sort&compress}

\usepackage[utf8]{inputenc}
\usepackage{color,graphicx,framed}
\usepackage{multicol,multirow}
\usepackage{amsmath,amssymb,amsthm}
\usepackage{tikz}
\usepackage[linesnumbered,ruled,vlined]{algorithm2e}
\usepackage{listings}
\usepackage{hyperref}
\usepackage[colorinlistoftodos,prependcaption,textsize=tiny]{todonotes}
\usepackage[T1]{fontenc}
\usepackage[english]{babel}
\usepackage{multicol}
\usepackage{ifthen}
\usepackage{caption}
\usepackage{subcaption}
\usepackage{geometry,mdwlist}

\definecolor{blue2b}{rgb}{0,0.1,0.3}
\definecolor{blue2}{rgb}{0,0.2,0.7}
\definecolor{red2}{rgb}{0.6,0.1,0.0}
\definecolor{green2}{rgb}{0.1,0.4,0.0}
\definecolor{yel2}{rgb}{0.3,0.2,0.0}
\definecolor{purple2}{rgb}{0.5,0.0,0.5}
\definecolor{blue3}{rgb}{0.65,0.85,1.0}
\definecolor{red3}{rgb}{1.0,0.7,0.5}
\definecolor{green3}{rgb}{0.8,1.0,0.7}
\definecolor{yel3}{rgb}{1.0,1.0,0.7}
\definecolor{grey3}{rgb}{0.95,0.95,0.95}
\definecolor{gray3}{rgb}{0.95,0.95,0.95}

\newboolean{debug}
\setboolean{debug}{false}

\renewcommand{\emph}{\textit}%

\newcommand{\vect}[1]{{\mathbf{#1}}}
\newcommand{\func}[1]{{\mathsf{#1}}}

\newtheorem{definition}{Definition}
\newtheorem{theorem}{Theorem}

\usepackage{chngcntr}
\counterwithout{footnote}{section}
\newtheorem{example}{Example}

\newlength\mylen

\journal{Information Sciences}
\begin{document}

\begin{frontmatter}

\title{Genie: A new, fast, and outlier-resistant hierarchical~clustering~algorithm}

\author[affil1,affil2]{Marek Gagolewski\corref{cor1}}
\author[affil2,affil3]{Maciej Bartoszuk}
\author[affil1,affil3]{Anna Cena}

\address[affil1]{Systems Research Institute, Polish Academy of Sciences\\%
ul.~Newelska 6, 01-447 Warsaw, Poland}
\address[affil2]{Faculty of Mathematics and Information Science,
Warsaw University of Technology\\%
ul.~Koszykowa 75, 00-662 Warsaw, Poland}
\address[affil3]{International PhD Studies Program,\\%
Institute of Computer Science, Polish Academy of Sciences}

\cortext[cor1]{Corresponding author; Email: marek.gagolewski@pw.edu.pl}

\begin{abstract}{\small
The time needed to apply a hierarchical clustering algorithm
is most often dominated by the number of computations of a pairwise
dissimilarity measure. Such a constraint, for larger data sets,
puts  at a disadvantage the use of all the classical linkage
criteria but the single linkage one. However, it is known that the single
linkage clustering algorithm is very sensitive to outliers, produces highly
skewed dendrograms, and therefore usually does not reflect the true
underlying data structure -- unless the clusters are well-separated.
To overcome its limitations, we propose a new hierarchical clustering linkage
criterion called Genie. Namely, our algorithm links two clusters in such a way that a chosen
economic inequity measure (e.g., the Gini- or Bonferroni-index) of the cluster
sizes does not increase drastically above a given threshold. The presented
benchmarks indicate a high practical usefulness of the introduced method:
it most often outperforms the Ward or average linkage in terms of the clustering quality
while retaining the single linkage speed.
The Genie algorithm is easily parallelizable and thus may be run
on multiple threads to speed up its execution further on.
Its memory overhead is small: there is no need to precompute the complete
distance matrix to perform the computations in order to obtain a desired
clustering. It can be applied on arbitrary spaces equipped with a dissimilarity
measure, e.g., on real vectors, DNA or protein sequences, images, rankings,  informetric data, etc.
A reference implementation of the algorithm has been included
in the open source \texttt{genie} package for~R.
[\textit{See \url{https://genieclust.gagolewski.com} for a new implementation --
available for both R and Python.}]
}

\noindent
{\footnotesize\color{blue}Please cite this paper as:
Gagolewski M., Bartoszuk M., Cena A.,
Genie: A new, fast, and outlier-resistant hierarchical~clustering~algorithm,
\textit{Information Sciences} \textbf{363}, 2016, pp.~8--23, doi:10.1016/j.ins.2016.05.003.}

\end{abstract}


\end{frontmatter}

\ifthenelse{\boolean{debug}}{\listoftodos[Notes]}{}

\ifthenelse{\boolean{debug}}{
\tableofcontents
}{}

\section{Introduction}

Cluster analysis, compare~\cite{TibEtAll:elementsstat}, is one of the most
commonly applied unsupervised machine learning techniques.
Its aim is to automatically discover an underlying structure of a given data set
$\mathcal{X}=\{\vect{x}^{(1)},\vect{x}^{(2)},\dots, \vect{x}^{(n)}\}$
in a form of a partition of its elements: disjoint and nonempty subsets are
determined in such a way that observations within each group are ``similar''
and entities in distinct clusters ``differ'' as much as possible from each other.
This contribution focuses on classical hierarchical clustering algorithms
\cite{CaiETAL2014:hierclust,Dasgupta2002:hier} which
determine a sequence of nested partitions, i.e., a whole hierarchy of data set subdivisions
that may be cut at an arbitrary level
and may be computed based on a pairwise dissimilarity measure
$\mathfrak{d}: \mathcal{X}\times \mathcal{X}\to[0,\infty]$ that fulfills \textit{very mild assumptions}:
(a) $\mathfrak{d}$ is symmetric, i.e., $\mathfrak{d}(\vect{x}, \vect{y}) = \mathfrak{d}(\vect{y}, \vect{x})$
and (b) $(\vect{x}=\vect{y}) \implies \mathfrak{d}(\vect{x}, \vect{y}) = 0$
for any $\vect{x}, \vect{y}\in \mathcal{X}$.
This group of clustering methods is often opposed to -- among others -- partitional schemes
which require the number of output clusters to be set up in advance:
these include the $k$-means, $k$-medians, $k$-modes, or $k$-medoids algorithms
\cite{XuWunsch:clustering,MacQueen1967:kmeans,JiangETAL2016:kmodes,ZahraETAL2015:kmeans}
and fuzzy clustering schemes \cite{Bezdek1981:fcm,PedryczWaletzky1997:fuzclustsup,%
PedryczBargiela2002:granclust,Pedrycz1996:conditionalfcm},
or the BIRCH (balanced iterative reducing and clustering using hierarchies)
method \cite{ZhangETAL1996:BIRCH} that works on real-valued vectors only.

In the large and big data era, one often is faced with the need to cluster
data sets of considerable sizes, compare, e.g., \cite{HalimETAL2015:largeclust}.
If the $(\mathcal{X},\mathfrak{d})$ space is ``complex'', we observe that the
run-times of hierarchical clustering algorithms
are dominated by the cost of pairwise distance (dissimilarity measure) computations. This is the case of, e.g.,
DNA sequences or ranking clustering, where elements in $\mathcal{X}$ are encoded as integer vectors,
often of considerable lengths. Here, one often relies on such computationally demanding metrics
as the Levenshtein or the Kendall one.
Similar issues appear in numerous other  application domains,
like pattern recognition, knowledge discovery, image processing, bibliometrics, complex networks, text mining,
or error correction, compare \cite{GomezETAL2015:hiernet,CaiETAL2014:hierclust,%
DinuIonescu2012:clusteringcloseststring,FerreiraZhao2016:clusttime,DimitrovskiETAL2016:bagclust}.

In order to achieve greater speed-up, most hierarchical clustering algorithms are applied on
a precomputed distance matrix, $(d_{i,j})_{i < j}$,
$d_{i,j}=\mathfrak{d}(\vect{x}^{(i)}, \vect{x}^{(j)})$,
so as to avoid determining the dissimilarity measure for each unique (unordered) pair
more than once. This, however, drastically limits the size of an input data set
that may be processed. Assuming that $d_{i,j}$ is represented with the
64-bit floating-point (IEEE double precision) type, already the case of $n=100{,}000$ objects is way beyond
the limits of personal computers popular nowadays: the sole distance matrix
would occupy ca.~40GB of available RAM. Thus, for ``complex'' data domains,
we must require that the number of calls to $\mathfrak{d}$ is kept as small as
possible. This practically disqualifies all popular hierarchical clustering
approaches other than the \emph{single linkage} criterion,
for which there is a fast $O(n^2)$-time and $O(n)$-space algorithm, see
\cite{Mullner2011:fastclusteralg,Mullner2013:fastcluster,XuWunsch:clustering},
that requires each $d_{i,j}$, $i<j$, to be computed exactly once,
i.e., there are precisely $(n^2-n)/2$ total calls to $\mathfrak{d}$.

Nevertheless, the single linkage criterion is not eagerly used by practitioners.
This is because it is highly sensitive to observations laying far away
from clusters' boundaries (e.g., outliers). Because of that, it
tends to produce highly skewed dendrograms: at its higher levels one often finds a
single large cluster and a number of singletons.

In order to overcome the limitations of the single linkage scheme,
in this paper we propose a new linkage criterion called \emph{Genie}. It not only
produces high-quality outputs (as compared, e.g., to the Ward and average linkage
criteria) but  is also relatively fast to compute.
The contribution is set out as follows. In the next section,
we review some basic properties of hierarchical clustering algorithms
and introduce the notion of an inequity index, which can be used to compare
some aspects of the quality of clusterings obtained by means of different
algorithms. The new linkage criterion, together with its evaluation
on diverse benchmark sets, is introduced in Section~\ref{Sec:algorithm}.
In Section~\ref{Sec:implementation} we propose an algorithm to compute
the introduced clustering scheme and test its time performance.
Please note that a reference implementation of the Genie method has been included
in the \texttt{genie} package\footnote{See \url{https://genieclust.gagolewski.com} for a new implementation --
available for both R and Python.} for R \cite{Rproject:home}.
Finally, Section~\ref{Sec:conclusions} concludes the paper and provides
future research directions.

\section{A discussion on classical linkage criteria}\label{Sec:single}

While a hierarchical clustering algorithm is being computed
on a given data set $\mathcal{X}=\{\vect{x}^{(1)},\vect{x}^{(2)},\dots, \vect{x}^{(n)}\}$,
there are $n-j$ clusters at the $j$-th step of the procedure, $j=0,\dots,n-1$.
It is always true that $\mathcal{C}^{(j)}=\{C_1^{(j)},\dots,C_{n-j}^{(j)}\}$
with $C_u^{(j)}\cap C_v^{(j)} = \emptyset$ for $u\neq v$, $C_u^{(j)}\neq\emptyset$,
and $\bigcup_{u=1}^{n-j} C_u^{(j)} = \mathcal{X}$. That is, $\mathcal{C}^{(j)}$ is a partition of $\mathcal{X}$.

Initially, we have that $\mathcal{C}^{(0)}=\{\{\vect{x}^{(1)}\},\dots,\{\vect{x}^{(n)}\}\}$,
i.e., $C_i^{(0)}=\{\vect{x}^{(i)}\}$, $i=1,\dots,n$. In other words, each observation
is the sole member of its own cluster.
When proceeding from step $j-1$ to $j$, the clustering procedure decides
which of the two clusters $C_u^{(j-1)}$ and $C_v^{(j-1)}$, $u<v$,
are to be merged so that we get $C_i^{(j)}=C_{i}^{(j-1)}$ for $u\neq i<v$,
$C_u^{(j)}=C_u^{(j-1)}\cup C_v^{(j-1)}$, and
$C_i^{(j)}=C_{i+1}^{(j-1)}$ for $i>v$.
In the single (minimum) linkage scheme, $u$ and $v$ are such that:
\[
   \mathrm{arg}\min_{(u,v), u < v} \left(
\min_{\vect{a}\in C_u^{(j-1)}, \vect{b}\in C_v^{(j-1)}} \mathfrak{d}(\vect{a},\vect{b})
\right).
\]
On the other hand, the complete (maximum) linkage is based on:
\[
   \mathrm{arg}\min_{(u,v), u < v} \left(
\max_{\vect{a}\in C_u^{(j-1)}, \vect{b}\in C_v^{(j-1)}} \mathfrak{d}(\vect{a},\vect{b})
\right),
\]
the average linkage on:
\[
   \mathrm{arg}\min_{(u,v), u < v} \left(
   \frac{1}{|C_u^{(j-1)}| |C_v^{(j-1)}|}
\sum_{\vect{a}\in C_u^{(j-1)}, \vect{b}\in C_v^{(j-1)}} \mathfrak{d}(\vect{a},\vect{b})
\right),
\]
and Ward's (minimum variance) method, compare \cite{Mullner2011:fastclusteralg} and
also \cite{MurtaghLegendre2014:ward}, on:
\begin{eqnarray*}
   &&\mathrm{arg}\min_{(u,v), u < v}
   \frac{1}{{|C_u^{(j-1)}| + |C_v^{(j-1)}|}} \times \\
   &\times&
   \Bigg(
    \sum_{\vect{a}\in C_u^{(j-1)}, \vect{b}\in C_v^{(j-1)}} 2\mathfrak{d}^2(\vect{a},\vect{b})
   -  \frac{|C_v^{(j-1)}|}{|C_u^{(j-1)}|} \sum_{\vect{a},\vect{a}'\in C_u^{(j-1)}}{\mathfrak{d}^2(\vect{a},\vect{a}')}
   - \frac{|C_u^{(j-1)}|}{|C_v^{(j-1)}|}\sum_{\vect{b},\vect{b}'\in C_v^{(j-1)}} {\mathfrak{d}^2(\vect{b},\vect{b}')}
\Bigg),
\end{eqnarray*}
where $\mathfrak{d}$ is a chosen dissimilarity measure.

\subsection{Advantages of single-linkage clustering}

The main advantage behind the single linkage clustering lies in the
fact that its most computationally demanding part deals
with solving the minimum spanning tree (MST, see \cite{SL:MST}) problem,
compare, e.g., the classical Prim's \cite{Prim1957:MST}
or Kruskal's \cite{KruskalProof} algorithms as well as a comprehensive
historical overview by Graham and Hell \cite{GrahamHell1985:historymst}.
In particular, there is an algorithm \cite{Olson1995:parallelhierclust}
which can be run in parallel and which requires exactly $(n^2-n)/2$ distance
computations. Moreover, under certain assumptions on $\mathfrak{d}$
(e.g., the triangle inequality) and the space dimensionality, the Kruskal
algorithm may be modified so as to make use of some nearest-neighbor (NN)
search data structure which enables to speed up its computations further on
(the algorithm can also be run in parallel).
Having obtained the MST, a single linkage clustering
may then be computed very easily, compare \cite{Rohlf1973:mst}.

For the other three mentioned linkage schemes there
is, e.g., a quite general nearest-neighbor chains algorithm \cite{Murtagh:survey},
as well as some other methods which require that, e.g., $\mathcal{X}$ is
a subset of the Euclidean space $\mathbb{R}^d$, see \cite{Mullner2011:fastclusteralg,Mullner2013:fastcluster} for a survey.
Unfortunately, we observe that all these algorithms tend to quite frequently
refer to already computed dissimilarities; it may be shown that they use up
to $3n^2$ distance computations.
Practically, the only way to increase the performance of these algorithms  \cite{Olson1995:parallelhierclust}
is to pre-compute the whole distance matrix (more precisely, the elements either above or below its diagonal).
However, we already noted that such an approach is unusable for $n$ already of
moderate order of magnitude (``large data'').

\subsection{Drawbacks of single-linkage clustering}

Nevertheless, it may be observed that unless the underlying clusters are
very well-separated, the single linkage approach
tends to construct clusters of unbalanced sizes, often resulting -- at some
fixed dendrogram cut level -- in a single large cluster and a number of
singletons or ones with a very low cardinality.

\begin{figure}[htb!]
\centering
\includegraphics[width=12cm]{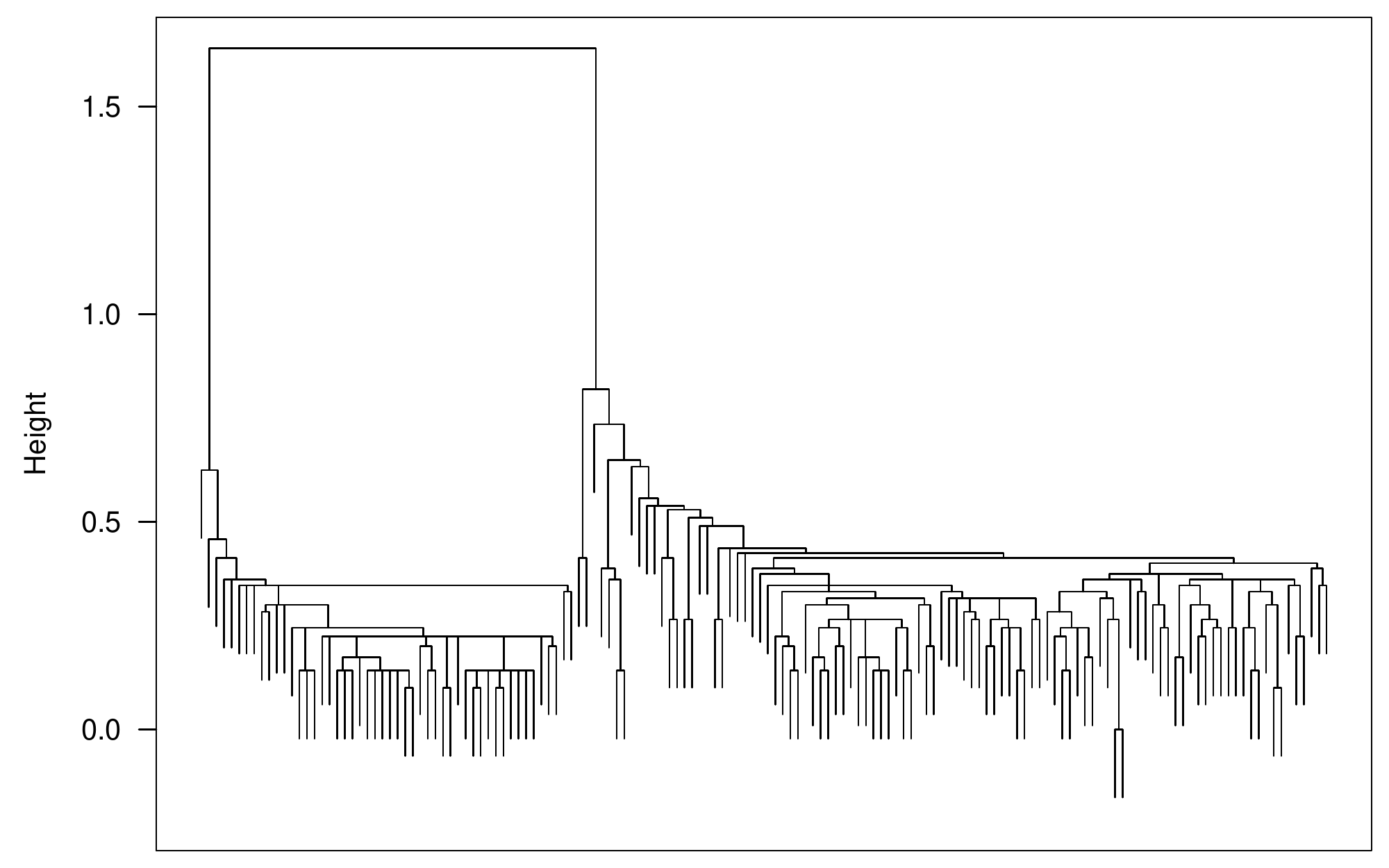}
  \caption{Dendrogram for the single linkage clustering of the \textit{Iris} data set.}
  \label{fig:single_iris_dendrogram}
\end{figure}

For instance, Figure~\ref{fig:single_iris_dendrogram} depicts a dendrogram
resulting in applying the single linkage clustering on the famous Fisher's \textit{Iris}
data set \cite{Fisher1936:iris} (available in the R \cite{Rproject:home}
\texttt{datasets} package, object name \texttt{iris}) with respect
to the Euclidean distance.
At the highest level, there are two clusters (50 observations corresponding
to \textit{iris setosa} and 100 both to \textit{virginica} and \textit{versicolor})
-- these two point groups are well-separated on a 4-dimensional plane.
Here, high skewness may be observed in the two subtrees,
e.g., cutting the left subtree at the height of 0.4 gives us a partition
consisting of three singletons and one large cluster of size 47.
These three observations lay slightly further away from the rest of the points.
When the $h=0.4$ cut of the whole tree is considered,
there are sixteen singletons, three clusters of size 2, one cluster of size 4,
and three large clusters of sizes 38, 39, and 47 (23 clusters in total).

\bigskip
In order to quantitatively capture the mentioned dendrogram \textit{skewness},
we may refer to the definition of an inequity (economic inequality, poverty) index, compare
\cite{GarciaLaprestaETAL2015:fuzzypoverty,BortotMarquesPereira2015:povertyemean,AristondoETAL2013:inequality}
and, e.g., \cite{KobusMilos2012:inequalitydecomp,Kobus2012:inequalitydecomp} for a different setting.

\begin{definition}
For a fixed $n\in\mathbb{N}$,
let $\mathcal{G}$ denote the set of all non-increasingly ordered $n$-tuples
with elements in the set of non-negative integers,
i.e., $\mathcal{G}=\{ (x_1,\dots,x_n)\in\mathbb{N}_0^n: x_1\ge\dots\ge x_n \}$.
Then $\func{F}:\mathcal{G}\to[0,1]$ is an inequity index,
whenever:
\begin{itemize}
   \item[(a)] it is Schur-convex, i.e.,
   for any $\vect{x},\vect{y}\in\mathcal{G}$ such that $\sum_{i=1}^n x_i=\sum_{i=1}^n y_i$,
   if it holds for all $i=1,\dots,n$ that $\sum_{j=1}^i x_j \le \sum_{j=1}^i y_j$, then
   $\func{F}(\vect{x})\le\func{F}(\vect{y})$,
   \item[(b)] $\inf_{\vect{x}\in\mathcal{G}} \func{F}(\vect{x})=0$,
   \item[(c)] $\sup_{\vect{x}\in\mathcal{G}} \func{F}(\vect{x})=1$.
\end{itemize}
\end{definition}

Notice that, under the current assumptions, if we restrict ourselves to the set of sequences with
elements summing up to $n$, the upper bound (meaning complete inequality) of each
inequity index is obtained for $(n,0,0,\dots,0)$
and the lower bound (no inequality) for the $(1,1,1,\dots,1)$ vector.

Every   inequity index like $\func{F}$ fulfills a crucial property
called the Pigou-Dalton principle
(also known as \emph{progressive transfers}).
Namely, for any $\vect{x}\in\mathcal{G}$, $i<j$, and $h>0$
such that $x_{i}-h \ge x_{i+1}$ and $x_{j-1}\ge x_j+h$, it holds that:
\[
\func{F}(x_1,\ldots,x_i,\ldots,x_j,\ldots,x_n) \ge \func{F}(x_{1},\ldots,x_{i}-h,\ldots,x_{j}+h,\ldots,x_{n}).
\]
In other words, any income transfer from a richer to a poorer
entity never increases the level of inequity.
Notably,  such measures of inequality of wealth distribution are not only of interest in economics:
it turns out, see \cite{BeliakovJames2015:unifyconsensus,BeliakovJamesNimmo2014:ecologicalconsensus},
that they can be related to ecological indices of evenness \cite{Pielou1969:introecology},
which aim to capture how evenly species' populations are distributed over a geographical
region, compare \cite{Camargo1993:dominance,Heip1974:evenness,Pielou1975:ecodiversity}
or measures of data spread \cite{Gagolewski2015:spread}.

Among notable examples of inequity measures we may find
the normalized Gini-index \cite{Gini1912:index}:
\begin{equation}\label{Eq:Gini}
\func{G}(\mathbf x) = \frac{\sum_{i=1}^{n-1} \sum_{j=i+1}^n |x_{i} - x_{j}|}{(n-1) \sum_{i=1}^n x_i}
\end{equation}
or the normalized Bonferroni-index \cite{Bonferroni1930:index}:
\begin{equation}
\func{B}(\mathbf x) = \frac{n}{n-1} \left(1 - \frac{\sum_{i=1}^n \frac{1}{n-i+1} \sum_{j=i}^n x_{j}}{ \sum_{i=1}^n x_i}\right).
\end{equation}

\bigskip
Referring back to the above motivational example,
we may say that there is often a high inequality between cluster sizes
in the case of the single linkage method.
Denoting by $c_i = |C_{i}^{(j)}|$ the size of the $i$-th cluster at the $j$-th iteration
of the algorithm, $\func{F}(c_{(n-j)},\dots,c_{(1)})$ tends to be very high
(here $c_{(i)}$ denotes the $i$-th smallest value in the given sequence,
obviously we always have $\sum_{i=1}^n c_i = n$.).

Figure~\ref{fig:gini_iris} depicts the Gini-indices of the cluster
size distributions as a function of the number of clusters
in the case of the \textit{Iris} data set and the Euclidean distance. The outputs of four
clustering methods are included: single, average, complete, and Ward linkage.
The highest inequality is of course observed in the case of the single linkage algorithm.
For instance, if the dendrogram is cut at height of $0.4$ (23 clusters in total,
their sizes are provided above), the Gini-index is as high as $\simeq 0.76$
(the maximum, $0.85$, is obtained for $10$ clusters).
In this example, the Ward  method keeps the Gini-index relatively low.
Similar behavior of hierarchical clustering algorithms may
be observed for other data sets.

\begin{figure}[htb!]
\centering
\includegraphics[width=12cm]{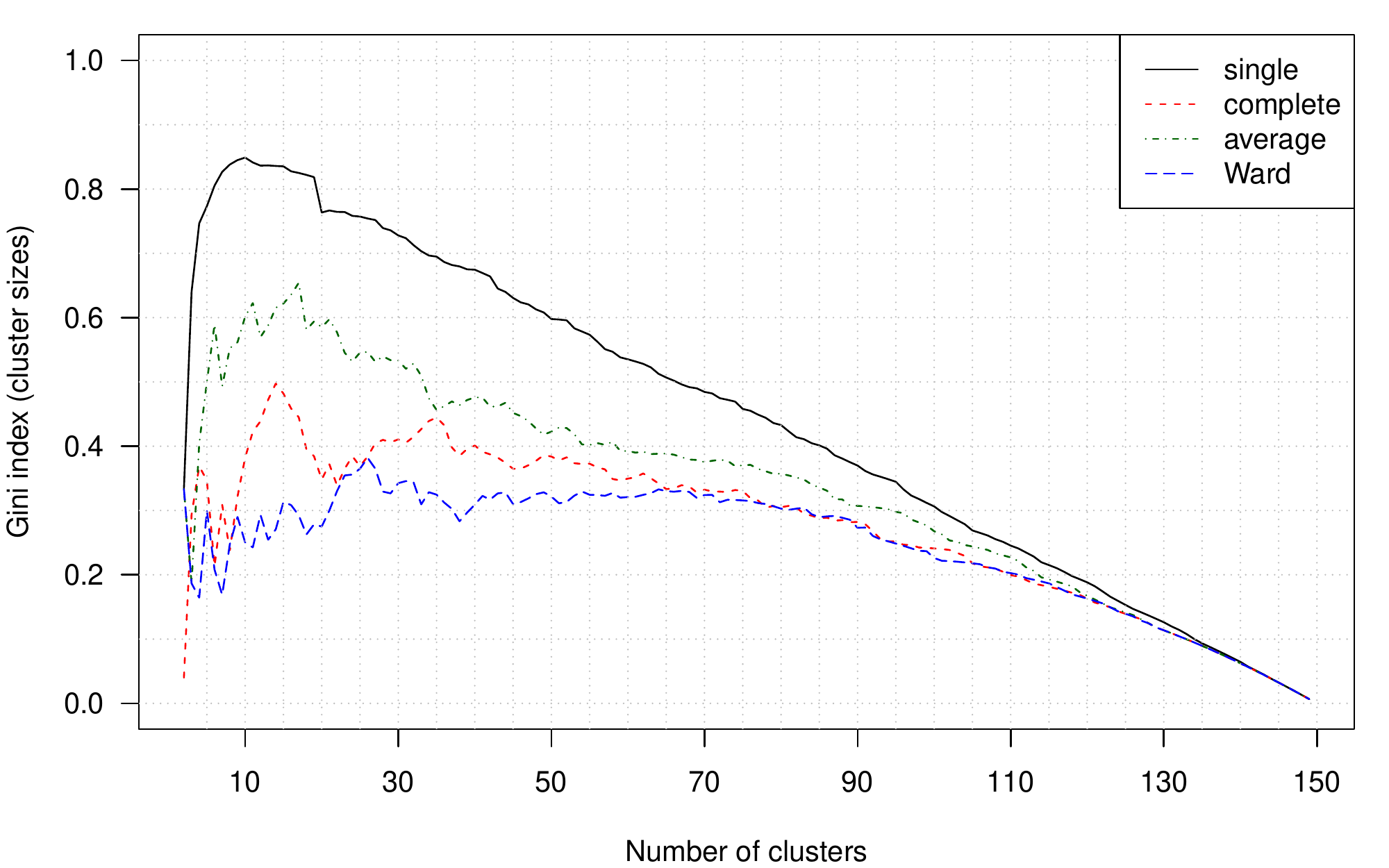}
  \caption{The Gini-indices of the cluster size distributions in the case of the \textit{Iris} data set:
  single, average, complete, and Ward linkage.}
  \label{fig:gini_iris}
\end{figure}

\section{The Genie algorithm and its evaluation}\label{Sec:algorithm}

\subsection{New linkage criterion}\label{Sec:NewLinkage}

In order to compensate the drawbacks of the single linkage scheme,
while retaining its simplicity, we propose the following linkage criterion
which from now on we refer to as \emph{the Genie algorithm}.
Let $\func{F}$ be a fixed inequity measure (e.g., the Gini-index)
and  $g\in(0,1]$ be some threshold.
At step $j$:
\begin{enumerate}
   \item if $\func{F}(c_{(n-j)}, \dots, c_{(1)}) \le g$, $c_i = |C_{i}^{(j)}|$,
apply the original single linkage criterion:
\[
   \mathrm{arg}\min_{(u,v), u < v} \left(
\min_{\vect{a}\in C_u^{(j)}, \vect{b}\in C_v^{(j)}} \mathfrak{d}(\vect{a},\vect{b})
\right),
\]
\item otherwise, i.e., if  $\func{F}(c_{(n-j)}, \dots, c_{(1)}) > g$,
restrict the search domain only to pairs of clusters such that one of them
is of the smallest size:
\[
   \mathrm{arg}\min_{\substack{(u,v), u < v,\\ |C_u^{(j)}| = \min_i |C_i^{(j)}|\text{ or }\\|C_v^{(j)}| = \min_i |C_i^{(j)}|}} \left(
\min_{\vect{a}\in C_u^{(j)}, \vect{b}\in C_v^{(j)}} \mathfrak{d}(\vect{a},\vect{b})
\right).
\]
\end{enumerate}
This modification prevents drastic increases of the chosen inequity measure
and forces early merges of small clusters with some other ones.
Figure~\ref{fig:gini_iris2} gives the cluster size distribution
(compare Figure~\ref{fig:gini_iris}) in case of the proposed algorithm
and the \textit{Iris} data set. Here, we used four different thresholds for
the Gini-index, namely, $0.3$, $0.4$, $0.5$, and $0.6$.
Of course, whatever the choice of the inequity index,
if $g=1$, then we obtain  the single linkage scheme.

\begin{figure}[htb!]
\centering
\includegraphics[width=12cm]{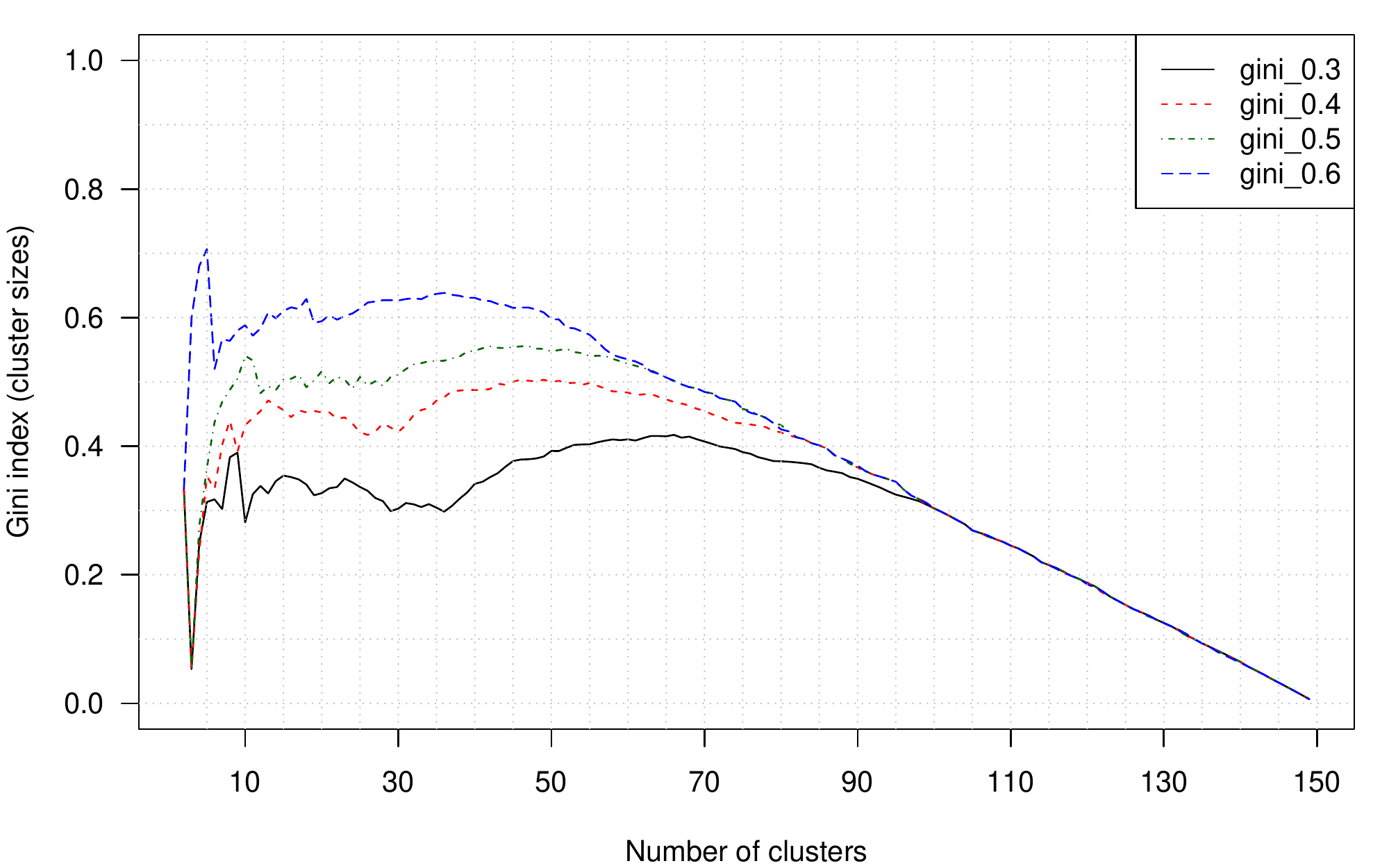}
  \caption{The Gini-indices of the cluster size distributions in the case of the \textit{Iris} data set:
  the Genie algorithm; the Gini-index thresholds are set to $0.3$, $0.4$, $0.5$, and $0.6$.}
  \label{fig:gini_iris2}
\end{figure}

On a side note, let us point out a small issue that may affect the way
the dendrograms resulting in applying the Genie algorithm are plotted.
As now the ``heights'' at which clusters are merged
are not being output in a nondecreasing order, they should somehow be adjusted
when drawing such diagrams.
Yet, the so-called reversals (inversions, departures from ultrametricity) are a well-known phenomenon, see
\cite{LegendreLegendre2003:numericalecology}, and may also occur in other linkages too
(e.g., the nearest-centroid one).

\subsection{Benchmark data sets description}\label{Sec:DSDescr}

In order to evaluate the proposed Genie linkage scheme,
we shall test it in various spaces (points in $\mathbb{R}^d$
for some $d$, images, character strings, etc.),
on balanced and unbalanced data of different shapes.
Below we describe the 29 benchmark data sets used,
21 of which are in the Euclidean space and the remaining 8 ones are non-Euclidean.
All of them are available for download and inspection
at \url{http://www.gagolewski.com/resources/data/clustering/}.
Please notice that most of the data sets have already been used in the literature
for verifying the performance of various other algorithms.

In each case below, $n$ denotes the number of objects and $d$ -- the space dimensionality
(if applicable).
For every data set its author(s) provided a vector of true (reference) cluster labels.
Therefore, we below denote with $k$ the true number of underlying clusters
(resulting dendrograms should be cut at this very level during the tests).
Moreover, we include the information on whether the reference clusters are of balanced sizes.
If this is not the case, the Gini-index of cluster sizes is reported.

\paragraph{Character strings}
\begin{itemize*}
\item \texttt{actg1} ($n=2500$, mean string length $d=99.9, k=20$, balanced),
\texttt{actg2} ($n=2500$, mean $d=199.9, k=5$, the Gini-index of the reference
cluster sizes is equal to 0.427),
\texttt{actg3} ($n=2500$, mean $d=250.2, k=10$, Gini-index 0.429) --
character strings (of varying lengths)
over the $\{a,c,t,g\}$ alphabet. First, $k$ random strings
(of identical lengths) were generated for the purpose of being cluster
centers. Each string in the data set was created by selecting a random cluster
center and then performing many Levenshtein edit operations (character insertions,
deletions, substitutions) at randomly chosen positions. For use
with the Levenshtein distance.
\item \texttt{binstr1} ($n=2500, d=100, k=25$, balanced),
\texttt{binstr2} ($n=2500, d=200, k=5$, Gini-index 0.432),
\texttt{binstr3} ($n=2500, d=250, k=10$, Gini-index 0.379) --
character strings (each of the same length $d$)
over the \{0,1\} alphabet. First, $k$ random strings were generated
for the purpose of being cluster centers. Each string in the data set
was created by selecting a random cluster center and then
modifying its digits at randomly chosen positions. For use
with the Hamming distance.
\end{itemize*}

\paragraph{Images}
These are the first 2000 digits from the famous MNIST database of
handwritten digits by Y.~LeCun et al.,
see \url{http://yann.lecun.com/exdb/mnist/}; clusters are approximately balanced.

\begin{itemize*}
\item \texttt{digits2k\_pixels} ($d=28\times 28$, $n=2000$, $k=10$) --
data consist of $28\times 28$ pixel images.
For testing purposes, we use the Hamming distance on corresponding monochrome
pixels (color value is marked with $1$ if the gray level is in the $(32,255]$
interval and $0$ otherwise).
\item \texttt{digits2k\_points} ($d=2$, $n=2000$, $k=10$) --
based on the above data set, we represent the contour of each digit as a set of points
in $\mathbb{R}^2$. Brightness cutoff of 64 was used to generate the data.
Each digit was shifted, scaled, and rotated if needed. For testing, we use
the Hausdorff (Euclidean-based) distance.
\end{itemize*}

\paragraph{SIPU benchmark data sets}
Researchers from the Speech and Image Processing Unit, School of Computing, University of
Eastern Finland prepared a list of exemplary benchmarks, which is
available at \url{http://cs.joensuu.fi/sipu/datasets/}. The data sets have already been used
in a number of papers. Because of the problems with computing the other linkages
in R as well as in Python, see the next section for discussion,
we chose only the data sets of sizes $\le$ 10000.
Moreover, we omitted the cases in which
all the algorithms worked flawlessly, meaning that the underlying clusters
were separated too well. In all the cases, we rely on the Euclidean distance.
\begin{itemize*}
\item \texttt{s1} ($n=5000, d=2, k=15$), \texttt{s2} ($n=5000, d=2, k=15$),
\texttt{s3} ($n=5000, d=2, k=15$), \texttt{s4} ($n=5000, d=2, k=15$) --
S-sets \cite{FrantiVirmajoki2006:ssets}. Reference clusters are more or less balanced.
\item \texttt{a1} ($n=3000, d=2, k=20$),
      \texttt{a2} ($n=5250, d=2, k=35$),
      \texttt{a3} ($n=7500, d=2, k=50$) -- A-sets \cite{KarkkainenFranti2002:asets}.
      Classes are fully balanced.
\item
\texttt{g2-2-100} ($n=2048, d=2, k=2$),
\texttt{g2-16-100} ($n=2048, d=16, k=2$),
\texttt{g2-64-100} ($n=2048, d=64, k=2$)
-- G2-sets. Gaussian clusters of varying dimensions, high variance.
Clusters  are fully balanced.

\item \texttt{unbalance} ($n=6500, d=2, k=8$). Unbalanced clusters, the Gini-index
of reference cluster sizes is 0.626.
\item \texttt{Aggregation} ($n=788, d=2, k=7$) \cite{GionisETAL2007:clustagg}. Gini-index 0.454.
\item \texttt{Compound} ($n=399, d=2, k=6$) \cite{Zahn1971:gestalt}. Gini-index  0.440.
\item \texttt{pathbased} ($n=300, d=2, k=3$) \cite{ChangYeung2008:pathbased}. Clusters  are more or less balanced.
\item \texttt{spiral} ($n=312, d=2, k=3$) \cite{ChangYeung2008:pathbased}. Clusters  are more or less balanced.
\item \texttt{D31} ($n=3100, d=2, k=31$) \cite{VeenmanETAL2002:maxvar}. Clusters  are fully balanced.
\item \texttt{R15} ($n=600, d=2, k=15$) \cite{VeenmanETAL2002:maxvar}. Clusters  are fully balanced.
\item \texttt{flame} ($n=240, d=2, k=2$) \cite{FuMedico2005:flame}. Gini-index  0.275.
\item \texttt{jain} ($n=373, d=2, k=2$) \cite{JainLaw2005:dilemma}. Gini-index  0.480.
\end{itemize*}

\paragraph{Iris}
The Fisher's \textit{Iris} \cite{Fisher1936:iris} data set,
available in the R \cite{Rproject:home} \texttt{datasets} package. Again, the Euclidean distance is used.
\begin{itemize*}
\item \texttt{iris} ($d=4$, $n=150$, $k=3$) -- the original data set. Fully balanced clusters.
\item \texttt{iris5} ($d=4$, $n=105$, $k=3$) -- an unbalanced version of the above one,
in which we took only five last observations from the first group (\textit{iris setosa}).
Gini-index 0.429.
\end{itemize*}

\subsection{Benchmark results}

In order to quantify the degree of agreement between two $k$-partitions of a given set,
the notion of the FM-index  \cite{FowlkesMallows1983:FMindex} is very often used.

\begin{definition}
Let $\mathcal{C}=\{C_1,\dots,C_{k}\}$ and $\mathcal{C}'=\{C_1',\dots,C_{k}'\}$
be two $k$-partitions of the set $\{\vect{x}^{(1)},\dots,\vect{x}^{(n)}\}$.
The Fowlkes-Mallows (FM) index is given by:
\[
   \text{FM-index}(\mathcal{C},\mathcal{C}') = \frac{\sum_{i=1}^k \sum_{j=1}^k m_{i,j}^2 -n}%
   {\sqrt{\left(\sum_{i=1}^k \left(\sum_{j=1}^k m_{i,j}\right)^2 -n\right)
   \left(\sum_{j=1}^k \left(\sum_{i=1}^k m_{i,j}\right)^2 -n\right)}}\in[0,1],
\]
where $m_{i,j}=\left| C_{i}\cap C_{j}' \right|$.
\end{definition}

If the two partitions are equivalent (equal up to a permutation of subsets in one of the $k$-partitions), then
the FM-index is equal to $1$. Moreover, if each pair of objects
that appear in the same set in $\mathcal{C}$ appear in two different sets in $\mathcal{C}'$,
then the index is equal to $0$.

Let us compare the performance of the Genie algorithm (with $\func{F}$ set
to be the Gini-index; five different thresholds, $g\in\{0.2,0.3,\dots,0.6\}$, are used)
as well as the single, average, complete, and Ward linkage schemes
(as implemented in the \texttt{hclust()} function from the R \cite{Rproject:home}
package \texttt{stats}).
In order to do so, we compute the values of
$\text{FM-index}(\mathcal{C},\mathcal{C}')$, where
$\mathcal{C}$ denotes the vector of true (reference) cluster labels
(as described in Section~\ref{Sec:DSDescr}),
while $\mathcal{C}'$ is the clustering obtained by cutting at an appropriate level
the dendrogram returned by a hierarchical clustering algorithm being investigated.
However, please observe that for some benchmark data sets the distance
matrices consist of non-unique elements. As a result, the output of the
algorithms may vary slightly from call to call (this is the case of all the tested methods).
Therefore, we shall report the median FM-index across 10 runs of randomly
permuted observations in each benchmark set.

Table~\ref{Tab:Benchmarks} gives the FM-indices for the 9 clustering methods
and the 29 benchmark sets. Best results are marked with bold font.
Aggregated basic summary statistics (minimum, quartiles, maximum, arithmetic mean,
and standard deviation) for all the benchmark sets are provided
in Table~\ref{Tab:BenchmarksSummary}.
Moreover, Figure~\ref{Fig:Benchmarks} depicts violin plots of the FM-index distribution%
\footnote{A violin depicts a box-and-whisker plot (boxes range from the 1st to the 3rd quartile,
the median is marked with a white dot) together with a kernel density estimator
of the empirical distribution.}.

\begin{table}[p!]
\caption{\label{Tab:Benchmarks} FM-indices for the 29 benchmark sets and the 9 hierarchical clustering methods studied.}
\centering\small

\hspace*{-3em}
\begin{tabular}{l|rrrr|rrrrr}
  \hline
benchmark & single & complete & Ward & average  & gini\_0.2 & gini\_0.3 & gini\_0.4 & gini\_0.5 & gini\_0.6 \\
  \hline\hline

actg1 &  \color{red2} 0.222 & \color{green2} 0.997 & \color{green2} \textbf{0.998} & \color{green2} \textbf{0.998} & \color{green2} 0.996 & 0.941 & 0.947 & 0.905 & \color{red2} 0.624\\
actg2 &  \color{red2} 0.525 & \color{green2} \textbf{1.000} & \color{green2} \textbf{1.000} & \color{green2} \textbf{1.000} & \color{green2} 0.975 & \color{green2} 0.975 & \color{green2} 0.976 & \color{green2} \textbf{1.000} & \color{green2} \textbf{1.000}\\
actg3 &  \color{red2} 0.383 & \color{green2} \textbf{1.000} & \color{green2} \textbf{1.000} & \color{green2} \textbf{1.000} & 0.884 & \color{green2} 0.975 & \color{green2} 0.975 & \color{green2} \textbf{1.000} & 0.840\\
binstr1 &   \color{red2} 0.198 & 0.874 & 0.942 & 0.947 & \color{green2} \textbf{0.952} & 0.908 & 0.863 & 0.749 & \color{red2} 0.542\\
binstr2 &   \color{red2} 0.525 & \color{green2} 0.989 & \color{green2} \textbf{0.994} & \color{green2} \textbf{0.994} & 0.907 & 0.909 & \color{green2} 0.965 & \color{green2} 0.965 & 0.819\\
binstr3 &   \color{red2} 0.368 & 0.946 & \color{green2} 0.969 & \color{green2} \textbf{0.971} & 0.832 & 0.931 & 0.937 & 0.811 & \color{red2} 0.692\\
digits2k\_pixels & \color{red2} 0.315 & \color{red2} 0.310 & \color{red2} 0.561 & \color{red2} 0.326 & \color{red2} \textbf{0.584} & \color{red2} 0.473 & \color{red2} 0.473 & \color{red2} 0.368 & \color{red2} 0.321\\
digits2k\_points & \color{red2} 0.315 & \color{red2} 0.256 & \color{red2} 0.458 & \color{red2} 0.280 & \color{red2} \textbf{0.671} & \color{red2} 0.601 & \color{red2} 0.559 & \color{red2} 0.438 & \color{red2} 0.405\\
\hline
s1 &  \color{red2} 0.589 & \color{green2} 0.973 & \color{green2} 0.984 & \color{green2} 0.983 & \color{green2} \textbf{0.989} & \color{green2} \textbf{0.989} & \color{green2} \textbf{0.989} & \color{green2} \textbf{0.989} & \color{green2} \textbf{0.989}\\
s2 &  \color{red2} 0.257 & 0.807 & 0.912 & 0.918 & \textbf{0.921} & \textbf{0.921} & 0.791 & 0.804 & 0.767\\
s3 &  \color{red2} 0.257 & \color{red2} 0.548 & \color{red2} 0.699 & \color{red2} 0.636 & \textbf{0.708} & \color{red2} 0.690 & \color{red2} 0.610 & \color{red2} 0.609 & \color{red2} 0.559\\
s4 &  \color{red2} 0.257 & \color{red2} 0.468 & \color{red2} 0.585 & \color{red2} 0.546 & \color{red2} \textbf{0.644} & \color{red2} 0.620 & \color{red2} 0.563 & \color{red2} 0.529 & \color{red2} 0.482\\
a1 &  \color{red2} 0.564 & 0.920 & 0.918 & 0.929 & \textbf{0.940} & 0.905 & 0.901 & 0.849 & 0.776\\
a2 &  \color{red2} 0.480 & 0.911 & 0.924 & 0.936 & \color{green2} \textbf{0.951} & 0.925 & 0.903 & 0.843 & 0.703\\
a3 &  \color{red2} 0.449 & 0.919 & 0.939 & 0.945 & \color{green2} \textbf{0.958} & 0.940 & 0.923 & 0.836 & 0.743\\
g2-2-100 &  \textbf{0.707} & \color{red2} 0.586 & \color{red2} 0.598 & 0.706 & \color{red2} 0.601 & \color{red2} 0.602 & \color{red2} 0.637 & \color{red2} 0.648 & \color{red2} 0.648\\
g2-16-100 & 0.707 & 0.897 & \textbf{0.923} & 0.707 & 0.842 & \color{red2} 0.697 & \color{red2} 0.697 & \color{red2} 0.697 & 0.704\\
g2-64-100 & 0.707 & \color{green2} \textbf{1.000} & \color{green2} \textbf{1.000} & 0.707 & \color{green2} 0.999 & \color{green2} 0.999 & \color{green2} 0.999 & \color{green2} 0.999 & \color{green2} 0.999\\
unbalance & \color{green2} 0.999 & 0.775 & \color{green2} \textbf{1.000} & \color{green2} \textbf{1.000} & 0.723 & 0.730 & 0.775 & 0.844 & 0.911\\
Aggregation &  0.861 & 0.833 & 0.842 & \color{green2} \textbf{1.000} & \color{red2} 0.582 & \color{red2} 0.657 & 0.816 & 0.908 & 0.894\\
Compound &  0.830 & 0.855 & \color{red2} 0.653 & \textbf{0.862} & \color{red2} 0.638 & \color{red2} 0.649 & \color{red2} 0.637 & 0.708 & \textbf{0.889}\\
pathbased & \color{red2} 0.573 & \color{red2} 0.595 & \color{red2} 0.674 & \color{red2} 0.653 & \textbf{0.751} & \textbf{0.751} & \textbf{0.751} & \textbf{0.751} & \textbf{0.751}\\
spiral & \color{green2} \textbf{1.000} & \color{red2} 0.339 & \color{red2} 0.337 & \color{red2} 0.357 & \color{green2} \textbf{1.000} & \color{green2} \textbf{1.000} & \color{green2} \textbf{1.000} & \color{green2} \textbf{1.000} & \color{green2} \textbf{1.000}\\
D31 & \color{red2} 0.349 & 0.926 & 0.923 & 0.910 & \textbf{0.937} & 0.903 & 0.828 & 0.742 & \color{red2} 0.695\\
R15 & \color{red2} 0.637 & \color{green2} 0.980 & \color{green2} 0.983 & \color{green2} \textbf{0.990} & \color{green2} 0.987 & \color{green2} 0.987 & \color{green2} 0.987 & 0.823 & \color{red2} 0.637\\
flame &  0.730 & \color{red2} 0.623 & \color{red2} 0.624 & 0.731 & \color{green2} \textbf{1.000} & \color{green2} \textbf{1.000} & \color{green2} \textbf{1.000} & \color{green2} \textbf{1.000} & \color{green2} \textbf{1.000}\\
jain &   0.804 & 0.922 & 0.790 & 0.922 & \color{green2} \textbf{1.000} & \color{green2} \textbf{1.000} & \color{green2} \textbf{1.000} & \color{green2} \textbf{1.000} & \color{green2} \textbf{1.000}\\
iris &   0.764 & 0.769 & 0.822 & 0.841 & \textbf{0.923} & \textbf{0.923} & \textbf{0.923} & \textbf{0.923} & 0.754\\
iris5 &  \color{red2} 0.691 & \color{red2} 0.665 & 0.738 & 0.765 & 0.764 & 0.764 & 0.764 & \textbf{0.886} & \color{red2} 0.673\\
     \hline
\color{green2}FM-index $\ge 0.95$ &\color{red2}     2 &   7 & \color{green2}  9 & \color{green2}  9 &\color{green2} \bfseries 11 &   8 &  \color{green2} 9 &   8 &   6 \\
FM-index $\ge 0.9$ &\color{red2}2 &  13 & \color{green2} 16 & \color{green2} 16 & \color{green2}16 & \color{green2}\bfseries 18 &  15 &  11 &   7 \\
\color{red2}FM-index $< 0.7$ & \color{red2} 19 &   9 &   9 &\bfseries\color{green2}   6 & \bfseries\color{green2}  6 & \color{green2}  8 &  \color{green2} 7 &\bfseries\color{green2}   6 &  11 \\
\color{red2}FM-index $< 0.5$ & \color{red2} 12 &   4 & \color{green2}  2 &   3 & \bfseries\color{green2}  0 &\color{green2}   1 & \color{green2}  1 & \color{green2}  2 &   3 \\
\hline
\end{tabular}
\end{table}

\begin{table}[p!]
\caption{\label{Tab:BenchmarksSummary}  Basic summary statistics of the FM-index distribution over the 29 benchmark sets.}
\centering\small

\hspace*{-1em}
\begin{tabular}{l|rrrr|rrrrr}
  \hline
 & single & complete & Ward & average & gini\_0.2 & gini\_0.3 & gini\_0.4 & gini\_0.5 & gini\_0.6 \\
  \hline
  \hline
Min & \color{red2} 0.198 & \color{red2} 0.256 & \color{red2} 0.337 & \color{red2} 0.280 & \color{red2} \textbf{0.582} & \color{red2} 0.473 & \color{red2} 0.473 & \color{red2} 0.368 & \color{red2} 0.321\\
Q1 &  \color{red2} 0.349 & \color{red2} 0.623 & \color{red2} 0.674 & 0.707 & 0.723 & \color{red2} 0.697 & \textbf{0.751} & 0.742 & \color{red2} 0.648\\
Median & \color{red2} 0.564 & 0.874 & 0.918 & 0.918 & \textbf{0.921} & 0.909 & 0.901 & 0.843 & 0.751\\
Q3 &  0.707 & 0.946 & \color{green2} \textbf{0.983} & \color{green2} \textbf{0.983} & \color{green2} 0.975 & \color{green2} 0.975 & \color{green2} 0.975 & \color{green2} 0.965 & 0.894\\
Max & \color{green2} \textbf{1.000} & \color{green2} \textbf{1.000} & \color{green2} \textbf{1.000} & \color{green2} \textbf{1.000} & \color{green2} \textbf{1.000} & \color{green2} \textbf{1.000} & \color{green2} \textbf{1.000} & \color{green2} \textbf{1.000} & \color{green2} \textbf{1.000}\\
Mean &   \color{red2} 0.554 & 0.782 & 0.820 & 0.812 & \textbf{0.850} & 0.840 & 0.834 & 0.815 & 0.752\\
\hline
St.Dev. &   \color{red2} {0.235} & \color{red2} 0.225 &0.188 & \color{red2} 0.215 & \color{green2} \textbf{0.147} & \color{green2} 0.156 & \color{green2} 0.159 & 0.172 & 0.185\\
   \hline
\end{tabular}
\end{table}

\begin{figure}[p!]
\centering
\includegraphics[width=12cm]{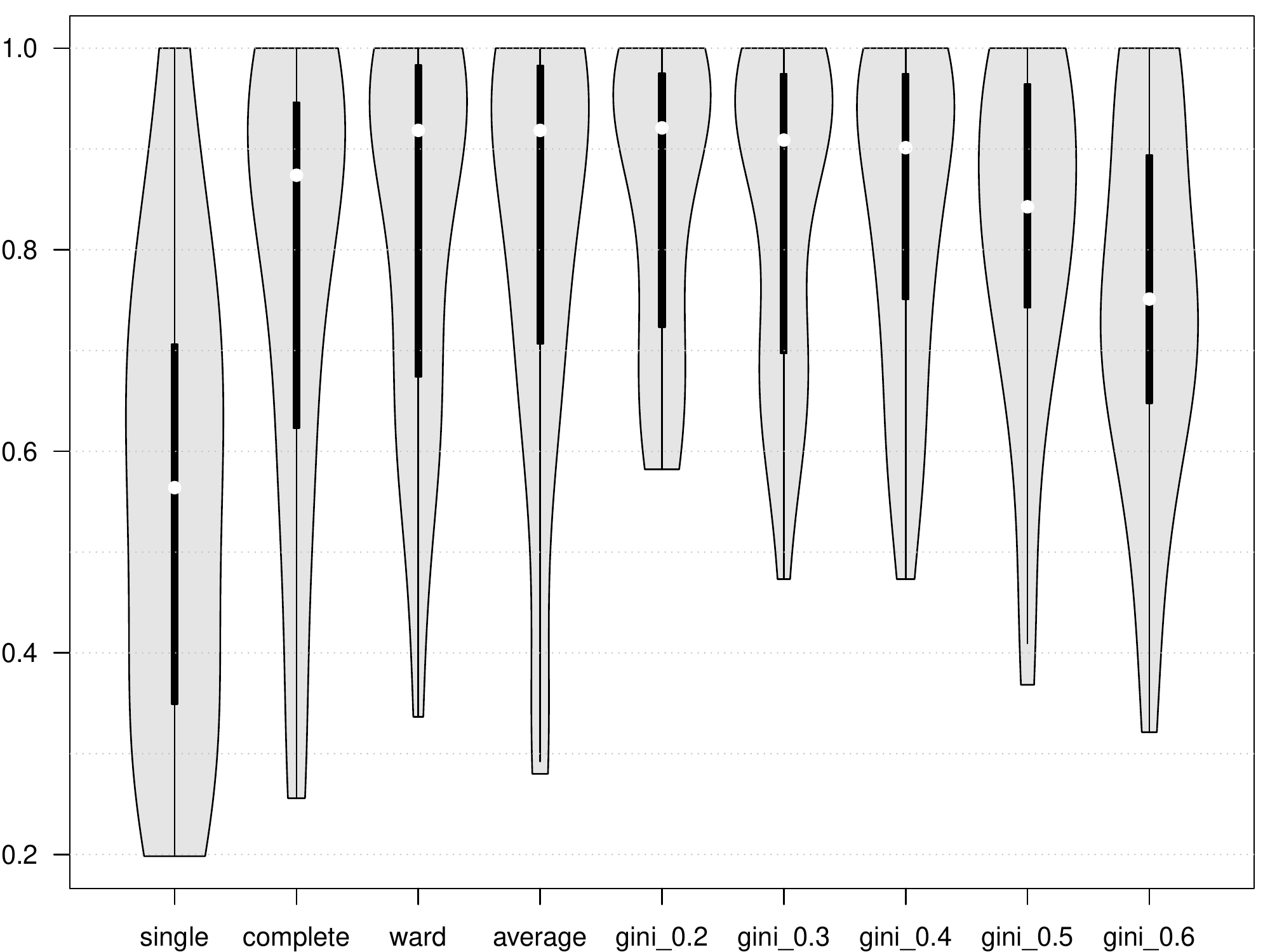}
\caption{\label{Fig:Benchmarks} Violin plots of the FM-index distribution over the  29 benchmark sets.}
\end{figure}

The highest mean and median FM scores were obtained for the Genie algorithm
with a threshold of $g=0.2$. This setting also leads to the best minimal (worst-case) FM-index.
A general observation is that all the tested Gini-index thresholds gave the lowest
variance in the FM-indices.

It is of course unsurprising that there is no free lunch in data clustering
-- no algorithm is perfect on all the data sets.
All the tested hierarchical clustering algorithms were far from perfect (FM $< 0.7$)
on the \texttt{digits2k\_pixels}, \texttt{digits2k\_points}, and
\texttt{s4} data sets.
However, in overall, the single linkage clustering is particularly bad
(except for the \texttt{unbalance} and \texttt{spiral} data sets).
Among the other algorithms, the complete linkage and the Genie algorithm for $g\ge 0.5$
give the lowest average and median FM-index. All the other methods
(Genie with thresholds of $g<0.5$, Ward, average linkage)
are very competitive. Also please keep in mind that for the Genie algorithm
with a low inequity index threshold we expect a loss in performance
for unbalanced clusters sizes

\begin{table}[t!]
\caption{\label{Tab:BenchmarksSummary2}  Basic summary statistics of the FM-index distribution over the 21 Euclidean benchmark sets.}
\centering\footnotesize

\hspace*{-1.5cm}
\begin{tabular}{rrrrrrrrrrrr}
  \hline
 & single & complete & ward & average & gini\_0.2 & gini\_0.3 & gini\_0.4 & gini\_0.5 & gini\_0.6 & BIRCH & k-means \\
  \hline
  Min & \color{red2}0.257 & \color{red2}0.339 & \color{red2}0.337 & \color{red2}0.357 & \color{red2}0.582 &\color{red2}\bf  0.602 &\color{red2} 0.563 &\color{red2} 0.529 &\color{red2} 0.482 &\color{red2} 0.350 & \color{red2}0.327 \\
  Q1 & \color{red2}0.480 & \color{red2}0.623 & \color{red2}0.674 & 0.707 & 0.723 & \color{red2}0.697 &\bf  0.751 & 0.742 &\color{red2} 0.695 & \color{red2}0.653 & 0.701 \\
  Median & \color{red2}0.691 & 0.833 & 0.842 & 0.862 &\bf  0.923 & 0.905 & 0.828 & 0.843 & 0.754 & 0.894 & 0.821 \\
  Q3 & 0.764 & 0.920 & 0.924 & 0.936 &\color{green2}\bf  0.987 &\color{green2}\bf  0.987 &\color{green2}\bf  0.987 & 0.923 & 0.911 & 0.924 & \color{green2}0.969 \\
  Max & \color{green2}\bf 1.000 &\color{green2}\bf  1.000 & \color{green2}\bf 1.000 &\color{green2}\bf  1.000 &\color{green2}\bf  1.000 & \color{green2}\bf 1.000 &\color{green2}\bf  1.000 &\color{green2}\bf  1.000 &\color{green2}\bf  1.000 &\color{green2}\bf  1.000 & \color{green2}\bf 1.000 \\
  Mean & 0.629 & 0.777 & 0.803 & 0.812 &\bf  0.850 & 0.841 & 0.833 & 0.828 & 0.789 & 0.801 & 0.816 \\
  \hline
  St.Dev. & \color{red2}0.224 & 0.187 & 0.177 & 0.172 & \color{green2}0.150 & \color{green2}0.146 &\color{green2} 0.145 & \color{green2}\bf 0.138 &\color{green2} 0.156 & 0.183 & 0.177 \\
   \hline
\end{tabular}
\end{table}

Finally, let us compare the performance of the 9 hierarchical clustering algorithms
as well as the $k$-means and BIRCH (\texttt{threshold=0.5}, \texttt{branching\_factor=10})
procedures (both implemented in the \texttt{scikit-learn} package for Python). Now we are of course limited only to data
in the Euclidean space, therefore the number of benchmark data sets reduces to 21.
Table~\ref{Tab:BenchmarksSummary2} gives basic summary statistics of the FM-index distributions.
We see that in this case the Genie algorithm ($g<0.5$) outperforms all the methods being compared too.

Taking into account our algorithm's out-standing performance
and -- as it shall turn out in the next section -- relatively low run-times
(especially on larger data sets and compared with the average or Ward linkage),
the proposed method may be recommended for practical use.

\section{Possible implementations of the Genie linkage algorithm}\label{Sec:implementation}

Having shown the high usability of the new approach, let us discuss
some ways to implement the Genie clustering method in very detail.
We have already stated that the most important part of computing the single
linkage algorithm consists of determining a minimal spanning tree
(this can be non-unique if there are pairs of objects with identical
dissimilarity degrees) of the complete undirected weighted graph
corresponding to objects in $\mathcal{X}$ and the pairwise dissimilarities.
It turns out that we have what follows.

\begin{theorem}\label{TheTheorem}
The Genie linkage criterion can be implemented based on an MST.
\end{theorem}
\begin{proof}[Sketch of the proof.]
By \cite[Principle~2]{Prim1957:MST}, in order to construct a minimal spanning tree,
it is sufficient to connect any two disconnected minimal spanning subtrees
via an edge of minimal weight and iterate such a process until a single connected
tree is obtained. As our linkage criterion (Section~\ref{Sec:NewLinkage})
always chooses such an edge, the proof is complete.
\end{proof}

\begin{figure}[htb!]
\centering
\begin{framed}
\begin{enumerate}
   \item[0.] Input: $\vect{x}^{(1)},\dots,\vect{x}^{(n)}$ -- $n$ objects,
   $g\in(0,1]$ -- inequity index threshold,\newline
$\mathfrak{d}$ -- a dissimilarity measure;
   \item[1.] \textit{ds} = DisjointSets($\{1\}, \{2\}, \dots, \{n\}$);
   \item[2.] \textit{m} = \textsf{MST}$(\vect{x}^{(1)},\dots,\vect{x}^{(n)})$;\hfill{\it /* see Theorem 1 */ }
   \item[3.] \textit{pq} = MinPriorityQueue<PQItem>($\emptyset$);

   \hfill{\it /* PQItem structure: (index1, index2, dist);

   \hfill pq returns the element with the smallest dist */}
   \item[4.] \textbf{for} each weighted edge $(i,j,d_{i,j})$ in \textit{m}:
   \begin{enumerate}
      \item[4.1.] \textit{pq}.\textsf{push}(PQItem($i$,$j$,$d_{i,j}$));
   \end{enumerate}
   \item[5.] \textbf{for} $j=1,2,\dots,n-1$:
   \begin{enumerate}
      \item[5.1.] \textbf{if} \textit{ds}.\textsf{compute\_inequity()} $\le$ $g$: \hfill{\it /* e.g., the Gini-index */ }
      \begin{enumerate}
         \item[5.1.1.] $t$ = \textit{pq}.\textsf{pop}(); \hfill{\it /* PQItem with the least dist */ }
      \end{enumerate}
      \item[] \textbf{else}:
      \begin{enumerate}
         \item[5.1.2.] $t$ = \textit{pq}.\textsf{pop\_conditional} $\Big(t$: \textit{ds}.\textsf{size}($t$.index1) = \textit{ds}.\textsf{min\_size}()\\
         \hspace*{9.15em}\textbf{or} \textit{ds}.\textsf{size}($t$.index2) = \textit{ds}.\textsf{min\_size}()$\Big)$;

         \hfill{\it /* PQItem with the least dist that fulfills the given logical condition */ }
      \end{enumerate}
      \item[5.2.] $s_1$ = \textit{ds}.\textsf{find\_set}($t$.index1);
      \item[5.3.] $s_2$ = \textit{ds}.\textsf{find\_set}($t$.index2); \hfill{\it /* assert: $s_1$ $\neq$ $s_2$ */}
      \item[5.4.] output ``linking ($s_1$, $s_2$)'';
      \item[5.5.] \textit{ds}.\textsf{link}($t$.index1, $t$.index2);
   \end{enumerate}
\end{enumerate}
\end{framed}
\caption{\label{Fig:Algorithm} A pseudocode for the Genie algorithm.}
\end{figure}

Please note that, by definition, the weighted edges appear in an MST
in no particular order -- for instance, the Prim \cite{Prim1957:MST} algorithm's output depends on the
permutation of inputs. Therefore, having established the above relation
between the Genie clustering and an MST, in Figure~\ref{Fig:Algorithm} we provide
a pseudocode of the algorithm that guarantees the right cluster merge order.
The procedure resembles the Kruskal \cite{KruskalProof} algorithm
and is fully concordant with our method's description in Section~\ref{Sec:NewLinkage}.

Observe that the very same algorithm can
be used to compute the single linkage clustering (in such a case, step 5.1 -- in which
we compute a chosen inequity measure -- as well as step 5.1.2
-- in which we test whether the inequity measure raises above a given threshold --
are never executed).

\begin{example}
Let us consider an exemplary data set consisting of 5 points in the real line:
$\vect{x}^{(1)}=0, \vect{x}^{(2)}=1, \vect{x}^{(3)}=3, \vect{x}^{(4)}=6, \vect{x}^{(5)}=10$
and let us set $\mathfrak{d}$ to be the Euclidean distance.
The minimum spanning tree -- which, according to Theorem~\ref{TheTheorem}
gives us complete information needed to compute the resulting clustering --
consists of the following
edges: $\{\vect{x}^{(1)}, \vect{x}^{(2)}\}, \{\vect{x}^{(2)}, \vect{x}^{(3)}\},
\{\vect{x}^{(3)}, \vect{x}^{(4)}\}, \{\vect{x}^{(4)},\allowbreak \vect{x}^{(5)}\}$,
with weights $1, 2, 3,$ and $4$, respectively.

Let the inequity measure $\func{F}$ be the Gini-index. If the threshold $g$ is set to $1.0$
(the single linkage), the merge steps are as follows.
Please note that the order in which we merge the clusters
is simply determined by sorting the edges in the MST increasingly by weights,
just as in the Kruskal algorithm.

\begin{center}
   \begin{tabular}{lllr}
   \hline
   step & current partitioning & MST edge & Gini-index \\
   \hline
   \hline
   0. & $\{\vect{x}^{(1)}\}, \{\vect{x}^{(2)}\}, \{\vect{x}^{(3)}\}, \{\vect{x}^{(4)}\}, \{\vect{x}^{(5)}\}$ & $\{\vect{x}^{(1)}, \vect{x}^{(2)}\}_1$ &  0.0 \\
   1. & $\{\vect{x}^{(1)}, \vect{x}^{(2)}\}, \{\vect{x}^{(3)}\}, \{\vect{x}^{(4)}\}, \{\vect{x}^{(5)}\}$  & $\{\vect{x}^{(2)}, \vect{x}^{(3)}\}_2$& 0.2 \\
   2. & $\{\vect{x}^{(1)}, \vect{x}^{(2)}, \vect{x}^{(3)}\}, \{\vect{x}^{(4)}\}, \{\vect{x}^{(5)}\}$  & $\{\vect{x}^{(3)}, \vect{x}^{(4)}\}_3$& 0.4 \\
   3. & $\{\vect{x}^{(1)}, \vect{x}^{(2)}, \vect{x}^{(3)}, \vect{x}^{(4)}\}, \{\vect{x}^{(5)}\}$  & $\{\vect{x}^{(4)}, \vect{x}^{(5)}\}_4$& 0.6 \\
   4. & $\{\vect{x}^{(1)}, \vect{x}^{(2)}, \vect{x}^{(3)}, \vect{x}^{(4)}, \vect{x}^{(5)}\}$  & --- & --- \\
   \hline
   \end{tabular}
\end{center}

However, if the threshold is set to $g=0.3$, then when proceeding form step 2 to step 3,
we need to link $\{\vect{x}^{(4)}\}$ and $\{\vect{x}^{(5)}\}$
instead of $\{\vect{x}^{(1)}, \vect{x}^{(2)}, \vect{x}^{(3)}\}$ and $\{\vect{x}^{(4)}\}$ -- that is,
the merge  is based now on a different MST edge: $\{\vect{x}^{(4)}, \vect{x}^{(5)}\}_4$
instead of $\{\vect{x}^{(3)}, \vect{x}^{(4)}\}_3$.
Therefore, the resulting 2-partition will be different than the above one:
we obtain $\{\vect{x}^{(1)}, \vect{x}^{(2)}, \vect{x}^{(3)}\}, \{\vect{x}^{(4)}, \vect{x}^{(5)}\}$
(Gini-index = 0.2).
\end{example}

\subsection{Implementation details}

In this study, we decided to focus on the Gini-index. In order to make the algorithm
time-efficient, this inequity index must be computed incrementally.
Let $g_j$ denote the Gini-index value at the time when there are $n-j$ clusters,
$j=0,1,\dots,n-1$. Initially, when each cluster is of size $1$, the Gini-index
is of course equal to $0$, hence $g_0=0$. Assume that when proceeding from step $j-1$ to $j$
we link clusters of sizes $c_{s_1}$ and $c_{s_2}$. It can easily be shown that:
\[
   g_j = \frac{
   (n-j)\,n\, g_{j-1} +\sum_{i=1}^{n-j+1} \Big(
   |c_i-c_{s_1}-c_{s_2}|-|c_i-c_{s_1}|-|c_i-c_{s_2}|
   \Big) - c_{s_2} - c_{s_1} + |c_{s_1}-c_{s_2}|
   }{(n-j-1)\,n}.
\]
In other words, after each merge operation, the index is updated, which requires $O(n)$
operations instead of $O(n^2)$ when the index is recomputed from scratch
based on the original formula \eqref{Eq:Gini}.
On a side note, it might be shown that the Gini-index can be written as a
linear combination of order statistics (compare, e.g., the notion of an OWA operator),
but the use of such a definition would require sorting the cluster size vector
in each iteration or relying on an ordered search tree-like data structure.

Moreover, please note that the algorithm is based on a non-standard implementation
of the disjoint sets data structure. Basically, the required extension
that keeps track of cluster counts can be implemented quite easily.

The implementation of the \textsf{pop\_conditional}() method in
the priority queue \textit{pq} (we used a heap-based data structure)
can be done by means of an auxiliary queue, onto which elements not fulfilling
the logical condition given in step 5.1.2~are temporarily
moved. If, after a merge operation, the inequity index is still above the desired
threshold and the minimal cluster size did not change since the previous iteration,
the auxiliary priority queue is kept as-is and we continue to seek
a cluster of the lowest cardinality from the same place. Otherwise, the elements must go back
to \textit{pq} and the search must start over.

To conclude, the cost of applying the cluster merge procedure
($O(n^2)$ pessimistically for $g<1$ and $O(n\log n)$ for $g=1$)
is not dominated by the cost of determining an MST ($O(Cn^2)$ pessimistically,
under our assumptions this involves a function $C$ of data dimensionality
and reflects the cost of computing a given dissimilarity measure $\mathfrak{d}$).
Hence, the new clustering scheme shall give us run-times
that are comparable with the single linkage method.
It is worth noting that the time complexity $(\Theta(n^2))$
as well as the memory complexity $(\Theta(n))$ of the algorithm is optimal
(as far as the whole class of hierarchical clustering algorithms
is concerned, compare \cite{Mullner2011:fastclusteralg}).

On a side note, the NN-chains algorithm, which is suitable for solving -- among others --
the complete, average, and Ward linkage clustering also has a time complexity
of $O(n^2)$ \cite{Murtagh:survey,Mullner2011:fastclusteralg},
but, as we shall see in further on, it requires
computing at least 2--5 times more pairwise distances.

\bigskip
For the sake of comparison, let us study two different algorithms to determine an MST.

\paragraph{1. An MST algorithm based on Prim's one}
A first algorithm, sketched in \cite[Figure~8]{Olson1995:parallelhierclust},
is quite similar to the one by Prim \cite{Prim1957:MST}.
Its pseudocode is given in Figure~\ref{Fig:MSTPrim}.
Note that the algorithm guarantees that exactly $(n^2-n)/2$ pairwise distances
are computed. Moreover, the inner loop can be run in parallel.
In such a case, $M$ should be a vector of indices not yet in the MST
-- due to that the threads may have random access to such an array --
and step 6.2.3 should be moved to a separate loop -- so as to
a costly critical section is avoided.

\clearpage
\begin{figure}[htb!]
\centering
\begin{framed}
\begin{enumerate}
   \item[0.] Input: $\vect{x}^{(1)},\dots,\vect{x}^{(n)}$ -- $n$ objects,
$\mathfrak{d}$ -- a dissimilarity measure;
   \item[1.] \textit{F} = $(\infty, \dots, \infty)$ ($n$ times);\hfill{\it /* $F_j$ -- index of $j$'s nearest neighbor */}
   \item[2.] \textit{D} = $(\infty, \dots, \infty)$ ($n$ times);\hfill{\it /* $D_j$ -- distance to $j$'s nearest neighbor */}
   \item[3.] \textit{lastj} = 1;
   \item[4.] $m = \emptyset$; \hfill{\it /* a resulting MST */}
   \item[5.] $M = \{ 2,3,\dots,n \}$; \hfill{\it /* indices not yet in $m$ */}
   \item[6.] \textbf{for} $i=1,2,\dots,n-1$:
   \begin{enumerate}
      \item[6.1.] $\text{\it bestj} = 1$;
      \item[6.2.] \textbf{for} each $j\in M$:
      \begin{enumerate}
         \item[6.2.1.] $d = \mathfrak{d}(\vect{x}^{(\text{\it lastj})}, \vect{x}^{(j)})$;
         \item[6.2.2.] \textbf{if} $d < D_j$:
         \begin{enumerate}
            \item[6.2.2.1.] $D_j = d$;
            \item[6.2.2.2.] $F_j = \text{\it lastj}$;
         \end{enumerate}
         \item[6.2.3.] \textbf{if} $D_j < D_{\text{\it bestj}}$:\hfill  \hfill{\it /* assert: $D_{1} = \infty$ */}
         \begin{enumerate}
            \item[6.2.3.1.] $\text{\it bestj} = j$;
         \end{enumerate}
      \end{enumerate}
      \item[6.3.] $m = m\cup\{(F_{\text{\it bestj}}, \text{\it bestj}, D_{\text{\it bestj}})\}$; \hfill{\it /* add an edge to $m$ */}
      \item[6.4.] $M = M\setminus\{ \text{\it bestj} \}$; \hfill{\it /* now this index is in $m$ */}
      \item[6.5.] \textit{lastj} = \textit{bestj};
   \end{enumerate}
   \item[7.] \textbf{return} $m$;
\end{enumerate}
\end{framed}
\caption{\label{Fig:MSTPrim} A simple $(n^2-n)/2$ algorithm to determine an MST.}
\end{figure}

\paragraph{2. An MST algorithm based on Kruskal's one}
The second algorithm considered is based on the one by Kruskal \cite{KruskalProof}
and its pseudocode is given in Figure~\ref{Fig:MSTKruskal}.
It relies on a method called \textsf{getNextNearestNeighbor}(),
which fetches the index $j$ of the next not-yet considered nearest
(in terms of increasing $\mathfrak{d}$)
neighbor of an object at index $i$ having the property that $j>i$.
If such a neighbor does not exist anymore, the function returns $\infty$.
Please observe that in the prefetch phase the calls to \textsf{getNextNearestNeighbor()}
can be run in parallel.

\medskip
A na\"{i}ve implementation of the \textsf{getNextNearestNeighbor()} function
requires either $O(1)$ time and $O(n)$ memory for each object
(i.e., $O(n^2)$ in total -- the lists of all neighbors can be stored in $n$ priority queues, one
per each object) or $O(n)$ time and $O(1)$ memory (i.e., $O(n)$ in total -- no caching done at all).
As our priority is to retain total $O(n)$ memory use, the mentioned approach
is of course expected to have a much worse time performance than the Prim-based one.

However, now let us assume that a dissimilarity measure $\mathfrak{d}: \mathcal{X}\times\mathcal{X}\to[0,\infty]$
is in fact a pseudometric, i.e., it additionally fulfills the triangle inequality:
for any $\vect{x}, \vect{y}, \vect{z}\in \mathcal{X}$
we have $\mathfrak{d}(\vect{x}, \vect{y}) \leq \mathfrak{d}(\vect{x}, \vect{z}) + \mathfrak{d}(\vect{z}, \vect{y})$
-- such a setting often occurs in practice.

\begin{figure}[htb!]
\centering
\begin{framed}
\begin{enumerate}
   \item[0.] Input: $\vect{x}^{(1)},\dots,\vect{x}^{(n)}$ -- $n$ objects,
$\mathfrak{d}$ -- a dissimilarity measure;
   \item[1.] \textit{pq} = MinPriorityQueue<PQItem>($\emptyset$);

   \hfill{\it /* PQItem structure: (index1, index2, dist);

   \hfill pq returns the element with the smallest dist */}
   \item[2.] \textbf{for} $i=1,2,\dots,n-1$: \hfill{\it /* prefetch phase */}
   \begin{enumerate}
      \item[2.1.] $j =$ \textsf{getNextNearestNeighbor}($i$); \hfill{\it /* depends on $\vect{x}^{(1)},\dots,\vect{x}^{(n)}, \mathfrak{d}$ */}
      \item[2.2.] \textit{pq}.\textsf{push}(PQItem($i$,$j$,$\mathfrak{d}(\vect{x}^{(i)},\vect{x}^{(j)})$)); \hfill{\it /* assert: $i<j$ */}
   \end{enumerate}
   \item[3.] \textit{ds} = DisjointSets($\{1\}, \{2\}, \dots, \{n\}$);
   \item[4.] $m = \emptyset$;
   \item[5.] $i = 1$;
   \item[6.] \textbf{while} $i < n$: \hfill{\it /* merge phase */}
   \begin{enumerate}
      \item[6.1.] $t$ = \textit{pq}.\textsf{pop}(); \hfill{\it /* PQItem with the least dist */ }
      \item[6.2.] $s_1$ = \textit{ds}.\textsf{find\_set}($t$.index1);
      \item[6.3.] $s_2$ = \textit{ds}.\textsf{find\_set}($t$.index2);
      \item[6.4.] \textbf{if} $s_1\neq s_2$:
      \begin{enumerate}
         \item[6.4.1.] $i = i+1$;
         \item[6.4.2.] $m = m\cup\{(t.\text{index1}, t.\text{index2}, t.\text{dist})\}$;
         \item[6.4.3.] \textit{ds}.\textsf{link}($t$.index1, $t$.index2);
         \end{enumerate}
      \item[6.5.] $j =$ \textsf{getNextNearestNeighbor}($t.\text{index1}$);\hfill{\it /* assert: $t.\text{index1}<j$ */}
      \item[6.6.] \textbf{if} $j<\infty$:
      \begin{enumerate}
      \item[6.6.1.] \textit{pq}.\textsf{push}(PQItem($t.\text{index1}$,$j$,$\mathfrak{d}(\vect{x}^{(t.\text{index1})},\vect{x}^{(j)})$));
      \end{enumerate}
   \end{enumerate}
   \item[7.] \textbf{return} $m$;
\end{enumerate}
\end{framed}
\caption{\label{Fig:MSTKruskal} A \textsf{getNextNearestNeighbor}()-based algorithm to determine an MST.}
\end{figure}

In such a case, a significant speed up may be obtained by relying
on some nearest-neighbor search data structure
supporting queries like ``fetch a few nearest-neighbors of the $i$-th object
within the distance in range
$[r_\mathrm{min},r_\mathrm{max})$'', for some $r_\mathrm{min}<r_\mathrm{max}$.
Of course, it is widely known that -- due to the so-called curse of dimensionality,
compare \cite{BeyerEtAll1998:nnmeaningful,ChavezEtAll2001:searchingmetricspaces,%
RadavanovicETAL2010:hubs,Aggraval2001:searchingmetricspaces,Brin1995:nearneighbor}
-- there  is no general-purpose algorithm which always works
better than the na\"{i}ve method in spaces of high dimension.
Nevertheless, in our case, some modifications of a chosen data structure
may lead to improvements in time performance.

Our carefully tuned-up reference implementation (discussed below)
is based on a vantage point (VP)-tree, see \cite{Yianilos1993:vptree}.
The most important modifications applied are as follows.
\begin{itemize}
   \item Each tree node stores an information on the maximal object index
   that can be found in its subtrees. This speeds up the search
   for NNs of objects with higher indices.
   No distance computation are performed for a pair of indices $(i,j)$ unless $i<j$.
   \item Each tree node includes an information whether all its subtrees store
   elements from the same set in the disjoint sets data structure \textit{ds}. This Boolean flag is recursively
   updated during a call to \textsf{getNextNearestNeighbor}().
   Due to that, a significant number of tree nodes during the merge phase
   can be pruned.
   \item An actual tree query returns a batch of nearest-neighbors
   of adaptive size between 20 and 256 (the actual count is determined
   automatically according to how the underlying VP-tree prunes the
   child nodes during the search). The returned set of nearest-neighbors is
   cached in a separate priority queue, one per each input data point.
   Note that the size of the returned batch guarantees asymptotic
   linear total memory use.
\end{itemize}

\subsection{The \texttt{genie} package for R}

A reference implementation of the Genie algorithm has been
included in the \texttt{genie} package for R \cite{Rproject:home}.
This software is distributed under the open source GNU General Public License, version 3.
The package is available
for download at the official CRAN (Comprehensive R Archive Network) repository,
see \url{https://cran.r-project.org/web/packages/genie/},
and hence can be installed
from within an R session via a call to \texttt{install.packages("{}genie"{})}.
All the core algorithms have been developed in the C++11 programming language;
the R interface has been provided by means of the \texttt{Rcpp}
\cite{EddelbuettelFrancois2013:rcppbook} package.
What is more, we decided to rely on the OpenMP API in order to enable
multi-threaded computations.

A data set's clustering can be determined via a call to the \texttt{genie::hclust2()} function.
The \texttt{objects} argument, with which we provide a data set to be clustered,
may be a numeric matrix, a list of integer vectors, or an R character vector.
The dissimilarity measure is selected via the \texttt{metric} argument,
e.g., \texttt{"{}euclidean"{}}, \texttt{"{}manhattan"{}},  \texttt{"{}maximum"{}},
 \texttt{"{}hamming"{}},  \texttt{"{}levenshtein"{}},  \texttt{"{}dinu"{}}, etc.
The \texttt{thresholdGini} argument can be used to define the threshold for
the Gini-index (denoted with $g$ in Figure~\ref{Fig:Algorithm}).
Finally, the \texttt{useVpTree} argument
can be used to switch between the MST algorithms given
in Figures~\ref{Fig:MSTPrim} (the default) and \ref{Fig:MSTKruskal}.
For more details, please refer to the function's manual page (\texttt{?genie::hclust2}).

Here is an exemplary R session in which we
compute the clustering of the \texttt{flame} data set.

\begin{framed}
\begin{verbatim}
# load the `flame` benchmark data set,
# see http://www.gagolewski.com/resources/data/clustering/
data <- as.matrix(read.table(gzfile("flame.data.gz")))
labels <- scan(gzfile("flame.labels.gz"), quiet=TRUE)

# run the Genie algorithm, threshold g=0.2
result <- genie::hclust2(objects=data, metric="euclidean",
   thresholdGini=0.2)

# get the number of reference clusters
k <- length(unique(labels))

# plot the results
plot(data[,1], data[,2], col=labels, pch=cutree(result, k))

# compute the FM-index
as.numeric(dendextend::FM_index(labels, cutree(result, k),
   include_EV=FALSE))
## [1] 1
\end{verbatim}
\end{framed}

\subsection{Number of calls to the dissimilarity measure}

Let us compare the number of calls to the dissimilarity measure
$\mathfrak{d}$ required by different clustering algorithms.
The measures shall be provided relative to $(n^2-n)/2$, which is denoted with ``100\%''.

The benchmark data sets are generated as follows.
For a given $n$, $\sigma$, and $d$, $k=10$ cluster centers
$\boldsymbol\mu^{(1)}, \dots, \boldsymbol\mu^{(k)}$
are picked randomly from the uniform distribution on $[0,10]^d$.
Then, each of the $n$ observations is generated as $\boldsymbol\mu^{(j)}+\vect{y}$,
where $y_l$ for each $l=1,\dots,d$ is a random variate from the normal distribution
with expectation of $0$ and standard deviation of $\sigma$
and $j$ is a random number  distributed uniformly in $\{1,2,\dots,k\}$.
In other words, such a data generation method is more or less equivalent
to the one used in case of the \texttt{g2} data sets in the previous section.
Here, $\mathfrak{d}$ is set to be the Euclidean metric.

\begin{table}[thb!]
\caption{\label{Tab:PercentagesDist} Relative number of pairwise distance
computations (as compared to $(n^2-n)/2$) together with FM-indices (in parentheses).}
\centering
\hspace*{-1.5em}\small
\begin{tabular}{rrrrrrrr}  %
  \hline
$\sigma$ & $d$ & $n$ & gini\_0.3 & gini\_1.0 & complete & Ward & average \\
         &     &     &           & (single)  &          &      &         \\
  \hline
0.50 &   2 & 10000 & $\ast$4.8\% (0.76)  & 100\% (0.38) & 476\% (0.72) & 204\% (0.80) & 484\% (0.78) \\
0.50 &   5 & 10000 & $\ast$22.0\% (1.00) & 100\% (0.86) & 493\% (1.00) & 221\% (1.00) & 496\% (1.00) \\
1.50 &  10 & 10000 & $\ast$30.3\% (0.96) & 100\% (0.32) & 496\% (0.98) & 240\% (0.98) & 499\% (0.91) \\
1.50 &  15 & 10000 & $\ast$58.3\% (1.00) & 100\% (0.42) & 497\% (1.00) & 253\% (1.00) & 498\% (0.98) \\
1.50 &  20 & 10000 & $\ast$84.9\% (1.00) & 100\% (0.69) & 497\% (1.00) & 261\% (1.00) & 498\% (1.00) \\
3.50 & 100 & 10000 & $\ast$101.8\% (1.00) & 100\% (0.32) & 498\% (1.00) & 299\% (1.00) & 499\% (1.00) \\
5.00 & 250 & 10000 & $\ast$100.9\% (0.99) & 100\% (0.32) & 498\% (1.00) & 312\% (1.00) & 499\% (1.00) \\
   \hline
  1.5 &  10 & 100000 & $\ast$14.1\% (0.94)  & 100\% (0.32) & --- & 241\% (0.98) & --- \\
  3.5 & 100 & 100000 & $\ast$104.0\% (0.98) & 100\% (0.32) & --- & $\dagger$ 321\% (0.99) & --- \\
   \hline
\end{tabular}

\hspace*{6em}
\begin{flushleft}\footnotesize
($\ast$) -- a VP-tree used (Fig.~\ref{Fig:MSTKruskal})
to determine an MST; the $(n^2-n)/2$ algorithm (Fig.~\ref{Fig:MSTPrim}) could always be used instead.

($\dagger$) -- only one run of the experiment was conducted.
\end{flushleft}
\end{table}

In each test case (different choices of $n$, $d$, $\sigma$),
we randomly generated ten different data sets and averaged the resulting
FM-indices and relative numbers of calls to $\mathfrak{d}$.
Table~\ref{Tab:PercentagesDist} compares:
\begin{itemize}
   \item the Genie algorithm (\texttt{genie::hclust2()}, package version 1.0.0)
   based on each of the aforementioned MST algorithms (\texttt{useVpTree} equals
   either to \texttt{TRUE} or \texttt{FALSE}); the Gini-index threshold of
   $0.3$ and $1.0$, the latter is equivalent to the single linkage criterion,
   \item the Ward linkage (the \texttt{hclust.vector()} function from the
\texttt{fastcluster} 1.1.16 package \cite{Mullner2013:fastcluster}
-- this implementation works only in the case of the Euclidean distance
but uses $O(n)$ memory),
\item
the complete and average linkage (\texttt{fastcluster::hclust()}
-- they use an NN-chains-based algorithm, and require a pre-computed
distance matrix, therefore utilizing $O(n^2)$ memory).
\end{itemize}

We observe a positive impact of using a metric tree data structure (\texttt{useVpTree=TRUE})
in low-dimensional spaces.
In high-dimensional spaces, it is better to rely on the $(n^2-n)/2$ (Prim-like) algorithm.
Nevertheless, we observe that in high dimensional spaces
the relative number of calls to the dissimilarity measure is 2--5 times smaller than
in the case of other linkages. Please note that an NN-chains-based
version of the Ward linkage (not listed in the table; \texttt{fastcluster::hclust()})
gives similar results as the complete and average ones
and that its Euclidean-distance specific version (\texttt{fastcluster::hclust.vector()})
seems to depend on the data set dimensionality.

\subsection{Exemplary run-times}

Let us inspect exemplary run-time measurements of the \texttt{genie::hclust2()}
function (\texttt{genie} package version 1.0.0).
The measurements were performed on a laptop
with a Quad-core Intel(R) Core(TM) i7-4700HQ @ 2.40GHz CPU and
16 GB RAM. The computer was running Fedora 21 Linux (kernel 4.1.13-100)
and the gcc 4.9.2 compiler was used
(\texttt{-O2 -march=native} optimization flags).

\begin{table}[th!]
\caption{\label{Tab:Timings} Exemplary run-times (in seconds) for different thread numbers, $n=100{,}000$.}

\centering
\begin{tabular}{lllrrr}
\hline
      &                        &     & \multicolumn{3}{c}{Number of threads} \\
 data & MST algorithm & $g$ & 1 & 2 & 4  \\
\hline\hline
 $d=10$, $\sigma=1.5$ & Fig.~\ref{Fig:MSTKruskal} with a VP-tree & $0.3$ & 46.5          &   33.4        &  28.2 \\
                      & Fig.~\ref{Fig:MSTPrim} & $0.3$ & 91.5          &   59.9        &  44.8 \\
                      & Fig.~\ref{Fig:MSTKruskal} with a VP-tree     & $1.0$ & 32.0          &   19.1        &  13.4 \\
                      & Fig.~\ref{Fig:MSTPrim} & $1.0$ & 77.5          &   47.7        &  31.6 \\
               \hline
 $d=100$, $\sigma=3.5$& Fig.~\ref{Fig:MSTKruskal} with a VP-tree & $0.3$ & 1396         &  740         &  456   \\
                      & Fig.~\ref{Fig:MSTPrim} & $0.3$ & 743          &  413         &  293   \\
                      & Fig.~\ref{Fig:MSTKruskal} with a VP-tree & $1.0$ & 1385         &  717         &  454   \\
                      & Fig.~\ref{Fig:MSTPrim} & $1.0$ & 734          &  396         &  281   \\
\hline
\end{tabular}
\end{table}

Table~\ref{Tab:Timings} summarizes the results for $n=100{,}000$
and 1, 2, as well as 4 threads (set up via the \texttt{OMP\_THREAD\_LIMIT}
environmental variable).
The experimental data sets were generated in the same manner as above.
The reported times are minimums of 3 runs.
Note that the results also include the time needed to generate some additional
objects, so that the output is of the same form as the one generated
by the \texttt{stats::hclust()} function in~R.

We note that running 4 threads at a time (on a single multi-core CPU)
gives us a \hbox{2--3-fold} speed-up. Moreover, a VP-tree-based implementation (Figure~\ref{Fig:MSTKruskal}, \texttt{useVpTree=TRUE})
is twice as costly as the other one in spaces of high dimensions. However, in spaces
of low dimension it outperforms the $(n^2-n)/2$ approach (Figure~\ref{Fig:MSTPrim}).
Nevertheless, if a user is unsure whether he/she deals with a high- or
low-dimensional space and $n$ is of moderate order of magnitude, the simple approach
should rather be recommended, as it gives much more predictable timings.
This is why we have decided that the \texttt{useVpTree} argument should default to \texttt{FALSE}.

For a point of reference, let us note that a single test run of the Ward algorithm
(\texttt{fastclu\-ster::hclust.vector()}, single-threaded) for $n=100{,}000$,
$d=10$, $\sigma=1.5$ required 1452.8 seconds
and for $d=100$, $\sigma=3.5$ -- as much as 18433.7 seconds (which is almost 25 times
slower than the Genie approach).

\section{Conclusions}\label{Sec:conclusions}

We have presented a new hierarchical clustering linkage criterion
which is based on the notion of an inequity (poverty) index.
The performed benchmarks indicate that the proposed algorithm -- unless
the underlying cluster structure is drastically unbalanced -- works in
overall better not only than the widely used average and Ward linkage scheme,
but also than the $k$-means and BIRCH algorithms which can be applied on data
in the Euclidean space only.

Our method requires up to $(n^2-n)/2$ distance computations, which is
ca.~2--5 times less than in the case of the other popular linkage
schemes. Its performance is comparable with the single-linkage clustering.
As there is no need to store the full distance
matrix, the algorithm can be used to cluster larger (within one order
of magnitude) data sets than with the Ward and average linkage schemes.

Nevertheless, it seems that we have reached a kind of general limit of an input data set size
for ``classical'' hierarchical clustering, especially in case of multidimensional data.
Due to the curse of dimensionality, we do not have any nearest-neighbor search
data structures that would enable us to cluster data sets of sizes greater
than few millions of observations in a reasonable time span.
What is more, we should keep in mind that the lower bound for run-times
of all the hierarchical clustering methods is $\Omega(n^2)$ anyway.
However, let us stress that for smaller (non-big-data) samples, hierarchical clustering
algorithms are still very useful. This is due to the fact that they do not
require a user to provide the desired number of clusters in advance and that only a measure of objects'
dissimilarity -- fulfilling very mild properties --
must be provided in order to determine a data partition.

Further research on the algorithm shall take into account the effects of, e.g.,
choosing different inequity measures or relying on approximate nearest-neighbors
search algorithms and data dimension reduction techniques
on the clustering quality. Moreover, in a distributed environment, one may
consider partitioning subsets of input data individually and then rely on some
clustering aggregation techniques, compare, e.g., \cite{GionisETAL2007:clustagg}.

Finally, let us note that the Genie algorithm depends on a free parameter,
namely, the inequity index merge threshold, $g$.
The existence of such a tuning parameter is an advantage, as a user
may select its value to suit her/his needs.
In the case of the Gini-index,
we recommend the use of $g\in[0.2,0.5)$, depending on our knowledge of
the underlying cluster distribution. Such a choice led to outstanding results
during benchmark studies. However, we should keep in mind that if the threshold is too low,
the algorithm might have problems with correctly identifying clusters
of smaller sizes in case of unbalanced data.
On the other hand, $g$ cannot be too large, as the algorithm
might start to behave as the single-linkage method, which has a very poor performance.
A possible way to automate the choice of $g$ could consist of a few pre-flight runs
(for different thresholds) on a randomly chosen data sample,
a verification of the obtained preliminary clusterings' qualities,
and a choice of the best coefficient for the final computation.

\section*{Acknowledgments}
We would like to thank the Anonymous Reviewers for the constructive comments
that helped to significantly improve the manuscript's quality.
Moreover, we are indebted to Łukasz Błaszczyk for providing us with
the \texttt{scikit-learn} algorithms performance results.

This study was supported by the National Science Center, Poland,
research project 2014/13/D/HS4/01700.
Anna Cena and Maciej Bartoszuk would like to acknowledge the support by
the European Union from resources of the European Social Fund, Project PO KL
``Information technologies: Research and their interdisciplinary
applications'', agreement UDA-POKL.04.01.01-00-051/10-00
via the Interdisciplinary PhD Studies Program.

\bibliographystyle{model1b-num-names}

\begin{thebibliography}{60}
\expandafter\ifx\csname natexlab\endcsname\relax\def\natexlab#1{#1}\fi
\providecommand{\bibinfo}[2]{#2}
\ifx\xfnm\relax \def\xfnm[#1]{\unskip,\space#1}\fi
%
\bibitem[{Aggarwal et~al.(2001)Aggarwal, Hinneburg and
  Keimn}]{Aggraval2001:searchingmetricspaces}
\bibinfo{author}{C.C. Aggarwal}, \bibinfo{author}{A.~Hinneburg},
  \bibinfo{author}{D.A. Keimn}, \bibinfo{title}{On the surprising behavior of
  distance metric in high-dimensional space}, \bibinfo{journal}{Lecture Notes
  in Computer Science} \bibinfo{volume}{1973} (\bibinfo{year}{2001})
  \bibinfo{pages}{420--434}.
%
\bibitem[{Aristondo et~al.(2013)Aristondo, Garc\'{i}a-Lapresta, {Lasso de la
  Vega} and {Marques Pereira}}]{AristondoETAL2013:inequality}
\bibinfo{author}{O.~Aristondo}, \bibinfo{author}{J.~Garc\'{i}a-Lapresta},
  \bibinfo{author}{C.~{Lasso de la Vega}}, \bibinfo{author}{R.~{Marques
  Pereira}}, \bibinfo{title}{Classical inequality indices, welfare and illfare
  functions, and the dual decomposition}, \bibinfo{journal}{Fuzzy Sets and
  Systems} \bibinfo{volume}{228} (\bibinfo{year}{2013})
  \bibinfo{pages}{114--136}.
%
\bibitem[{Beliakov and James(2015)}]{BeliakovJames2015:unifyconsensus}
\bibinfo{author}{G.~Beliakov}, \bibinfo{author}{S.~James},
  \bibinfo{title}{Unifying approaches to consensus across different preference
  representations}, \bibinfo{journal}{Applied Soft Computing}
  \bibinfo{volume}{35} (\bibinfo{year}{2015}) \bibinfo{pages}{888--897}.
%
\bibitem[{Beliakov et~al.(2014)Beliakov, James and
  Nimmo}]{BeliakovJamesNimmo2014:ecologicalconsensus}
\bibinfo{author}{G.~Beliakov}, \bibinfo{author}{S.~James},
  \bibinfo{author}{D.~Nimmo}, \bibinfo{title}{Can indices of ecological
  evenness be used to measure consensus?}, in: \bibinfo{booktitle}{Proc. IEEE
  Intl. Conf. Fuzzy Systems'15}, \bibinfo{address}{Beijing, China},
  \bibinfo{year}{2014}, pp. \bibinfo{pages}{1--8}.
%
\bibitem[{Beyer et~al.(1998)Beyer, Goldstein, Ramakrishnan and
  Shaft}]{BeyerEtAll1998:nnmeaningful}
\bibinfo{author}{K.~Beyer}, \bibinfo{author}{J.~Goldstein},
  \bibinfo{author}{R.~Ramakrishnan}, \bibinfo{author}{U.~Shaft},
  \bibinfo{title}{{W}hen is nearest neighbor meaningful?}, in:
  \bibinfo{editor}{C.~Beeri}, \bibinfo{editor}{P.~Buneman} (Eds.),
  \bibinfo{booktitle}{Proc. ICDT}, \bibinfo{publisher}{Springer-Verlag},
  \bibinfo{year}{1998}, pp. \bibinfo{pages}{217--235}.
%
\bibitem[{Bezdek(1981)}]{Bezdek1981:fcm}
\bibinfo{author}{J.C. Bezdek}, \bibinfo{title}{Pattern Recognition with Fuzzy
  Objective Function Algorithms}, \bibinfo{publisher}{Springer},
  \bibinfo{year}{1981}.
%
\bibitem[{Bonferroni(1930)}]{Bonferroni1930:index}
\bibinfo{author}{C.~Bonferroni}, \bibinfo{title}{Elementi di statistica
  generale}, \bibinfo{publisher}{Libreria Seber}, \bibinfo{address}{Firenze},
  \bibinfo{year}{1930}.
%
\bibitem[{Bortot and {Marques
  Pereira}(2015)}]{BortotMarquesPereira2015:povertyemean}
\bibinfo{author}{S.~Bortot}, \bibinfo{author}{R.~{Marques Pereira}},
  \bibinfo{title}{On a new poverty measure constructed from the exponential
  mean}, in: \bibinfo{booktitle}{Proc. IFSA/EUSFLAT'15},
  \bibinfo{address}{Gi\'{j}on, Spain}, \bibinfo{year}{2015}, pp.
  \bibinfo{pages}{333--340}.
%
\bibitem[{Brin(1995)}]{Brin1995:nearneighbor}
\bibinfo{author}{S.~Brin}, \bibinfo{title}{Near neighbor search in large metric
  spaces}, in: \bibinfo{booktitle}{In Proceedings of the 21th International
  Conference on Very Large Data Bases}, \bibinfo{publisher}{Morgan Kaufmann
  Publishers}, \bibinfo{year}{1995}, pp. \bibinfo{pages}{574--584}.
%
\bibitem[{Cai et~al.(2014)Cai, Zhang, Tung, Dai and
  Hao}]{CaiETAL2014:hierclust}
\bibinfo{author}{R.~Cai}, \bibinfo{author}{Z.~Zhang}, \bibinfo{author}{A.K.
  Tung}, \bibinfo{author}{C.~Dai}, \bibinfo{author}{Z.~Hao}, \bibinfo{title}{A
  general framework of hierarchical clustering and its applications},
  \bibinfo{journal}{Information Sciences} \bibinfo{volume}{272}
  (\bibinfo{year}{2014}) \bibinfo{pages}{29--48}.
%
\bibitem[{Camargo(1993)}]{Camargo1993:dominance}
\bibinfo{author}{J.~Camargo}, \bibinfo{title}{Must dominance increase with the
  number of subordinate species in competitive interactions?},
  \bibinfo{journal}{Journal of {T}heoretical {B}iology} \bibinfo{volume}{161}
  (\bibinfo{year}{1993}) \bibinfo{pages}{537--542}.
%
\bibitem[{Chang and Yeung(2008)}]{ChangYeung2008:pathbased}
\bibinfo{author}{H.~Chang}, \bibinfo{author}{D.~Yeung}, \bibinfo{title}{Robust
  path-based spectral clustering}, \bibinfo{journal}{Pattern Recognition}
  \bibinfo{volume}{41} (\bibinfo{year}{2008}) \bibinfo{pages}{191--203}.
%
\bibitem[{Chavez et~al.(2001)Chavez, Navarro, Baeza-Yates and
  Marroquin}]{ChavezEtAll2001:searchingmetricspaces}
\bibinfo{author}{E.~Chavez}, \bibinfo{author}{G.~Navarro},
  \bibinfo{author}{R.~Baeza-Yates}, \bibinfo{author}{J.L. Marroquin},
  \bibinfo{title}{Searching in metric spaces}, \bibinfo{journal}{ACM Computing
  Surveys} \bibinfo{volume}{33} (\bibinfo{year}{2001})
  \bibinfo{pages}{273--321}.
%
\bibitem[{Dasgupta(2002)}]{Dasgupta2002:hier}
\bibinfo{author}{S.~Dasgupta}, \bibinfo{title}{Performance guarantees for
  hierarchical clustering}, in: \bibinfo{booktitle}{Proceedings of the
  Conference on Learning Theory}, \bibinfo{year}{2002}, pp.
  \bibinfo{pages}{351--363}.
%
\bibitem[{Dimitrovski et~al.(2016)Dimitrovski, Kocev, Loskovska and
  Džeroski}]{DimitrovskiETAL2016:bagclust}
\bibinfo{author}{I.~Dimitrovski}, \bibinfo{author}{D.~Kocev},
  \bibinfo{author}{S.~Loskovska}, \bibinfo{author}{S.~Džeroski},
  \bibinfo{title}{Improving bag-of-visual-words image retrieval with predictive
  clustering trees}, \bibinfo{journal}{Information Sciences}
  \bibinfo{volume}{329} (\bibinfo{year}{2016}) \bibinfo{pages}{851--865}.
%
\bibitem[{Dinu and Ionescu(2012)}]{DinuIonescu2012:clusteringcloseststring}
\bibinfo{author}{L.P. Dinu}, \bibinfo{author}{R.T. Ionescu},
  \bibinfo{title}{Clustering methods based on closest string via rank
  distance}, in: \bibinfo{booktitle}{14th Intl. Symp. Symbolic and Numeric
  Algorithms for Scientific Computing}, \bibinfo{publisher}{IEEE},
  \bibinfo{year}{2012}, pp. \bibinfo{pages}{207--213}.
%
\bibitem[{Eddelbuettel(2013)}]{EddelbuettelFrancois2013:rcppbook}
\bibinfo{author}{D.~Eddelbuettel}, \bibinfo{title}{Seamless \textsf{R} and
  \textsf{C++} Integration with \textsf{Rcpp}}, \bibinfo{publisher}{Springer},
  \bibinfo{address}{New York}, \bibinfo{year}{2013}.
%
\bibitem[{Ferreira and Zhao(2016)}]{FerreiraZhao2016:clusttime}
\bibinfo{author}{L.N. Ferreira}, \bibinfo{author}{L.~Zhao},
  \bibinfo{title}{Time series clustering via community detection in networks},
  \bibinfo{journal}{Information Sciences} \bibinfo{volume}{326}
  (\bibinfo{year}{2016}) \bibinfo{pages}{227--242}.
%
\bibitem[{Fisher(1936)}]{Fisher1936:iris}
\bibinfo{author}{R.~Fisher}, \bibinfo{title}{The use of multiple measurements
  in taxonomic problems}, \bibinfo{journal}{Annals of Eugenics}
  \bibinfo{volume}{7} (\bibinfo{year}{1936}) \bibinfo{pages}{179--188}.
%
\bibitem[{Fowlkes and Mallows(1983)}]{FowlkesMallows1983:FMindex}
\bibinfo{author}{E.~Fowlkes}, \bibinfo{author}{C.~Mallows}, \bibinfo{title}{A
  method for comparing two hierarchical clusterings}, \bibinfo{journal}{Journal
  of the American Statistical Association} \bibinfo{volume}{78}
  (\bibinfo{year}{1983}) \bibinfo{pages}{553--569}.
%
\bibitem[{Fränti and Virmajoki(2006)}]{FrantiVirmajoki2006:ssets}
\bibinfo{author}{P.~Fränti}, \bibinfo{author}{O.~Virmajoki},
  \bibinfo{title}{Iterative shrinking method for clustering problems},
  \bibinfo{journal}{Pattern Recognition} \bibinfo{volume}{39}
  (\bibinfo{year}{2006}) \bibinfo{pages}{761--765}.
%
\bibitem[{Fu and Medico(2007)}]{FuMedico2005:flame}
\bibinfo{author}{L.~Fu}, \bibinfo{author}{E.~Medico}, \bibinfo{title}{{FLAME},
  a novel fuzzy clustering method for the analysis of {DNA} microarray data},
  \bibinfo{journal}{BMC bioinformatics} \bibinfo{volume}{8}
  (\bibinfo{year}{2007}) \bibinfo{pages}{3}.
%
\bibitem[{Gagolewski(2015)}]{Gagolewski2015:spread}
\bibinfo{author}{M.~Gagolewski}, \bibinfo{title}{Spread measures and their
  relation to aggregation functions}, \bibinfo{journal}{European Journal of
  Operational Research} \bibinfo{volume}{241} (\bibinfo{year}{2015})
  \bibinfo{pages}{469--477}.
%
\bibitem[{Garc\'{i}a-Lapresta et~al.(2015)Garc\'{i}a-Lapresta, {Lasso de la
  Vega}, {Marques Pereira} and Urrutia}]{GarciaLaprestaETAL2015:fuzzypoverty}
\bibinfo{author}{J.~Garc\'{i}a-Lapresta}, \bibinfo{author}{C.~{Lasso de la
  Vega}}, \bibinfo{author}{R.~{Marques Pereira}}, \bibinfo{author}{A.~Urrutia},
  \bibinfo{title}{A new class of fuzzy poverty measures}, in:
  \bibinfo{booktitle}{Proc. of IFSA/EUSFLAT2015}, \bibinfo{address}{Gi\'{j}on,
  Spain}, \bibinfo{year}{2015}, pp. \bibinfo{pages}{1140--1146}.
%
\bibitem[{Gini(1912)}]{Gini1912:index}
\bibinfo{author}{C.~Gini}, \bibinfo{title}{Variabilità e mutabilità},
  \bibinfo{publisher}{C. Cuppini}, \bibinfo{address}{Bologna},
  \bibinfo{year}{1912}.
%
\bibitem[{Gionis et~al.(2007)Gionis, Mannila and
  Tsaparas}]{GionisETAL2007:clustagg}
\bibinfo{author}{A.~Gionis}, \bibinfo{author}{H.~Mannila},
  \bibinfo{author}{P.~Tsaparas}, \bibinfo{title}{Clustering aggregation},
  \bibinfo{journal}{ACM Transactions on Knowledge Discovery from Data}
  \bibinfo{volume}{1} (\bibinfo{year}{2007}) \bibinfo{pages}{4}.
%
\bibitem[{Gower and Ross(1969)}]{SL:MST}
\bibinfo{author}{J.~Gower}, \bibinfo{author}{G.~Ross}, \bibinfo{title}{Minimum
  spanning trees and single linkage cluster analysis},
  \bibinfo{journal}{Journal of the Royal Statistical Society. Series C (Applied
  Statistics)} \bibinfo{volume}{18} (\bibinfo{year}{1969})
  \bibinfo{pages}{54--64}.
%
\bibitem[{Graham and Hell(1985)}]{GrahamHell1985:historymst}
\bibinfo{author}{R.~Graham}, \bibinfo{author}{P.~Hell}, \bibinfo{title}{On the
  history of the minimum spanning tree problem}, \bibinfo{journal}{Annals of
  the History of Computing} \bibinfo{volume}{7} (\bibinfo{year}{1985})
  \bibinfo{pages}{43--57}.
%
\bibitem[{Gómez et~al.(2015)Gómez, Zarrazola, Yáñez and
  Montero}]{GomezETAL2015:hiernet}
\bibinfo{author}{D.~Gómez}, \bibinfo{author}{E.~Zarrazola},
  \bibinfo{author}{J.~Yáñez}, \bibinfo{author}{J.~Montero}, \bibinfo{title}{A
  divide-and-link algorithm for hierarchical clustering in networks},
  \bibinfo{journal}{Information Sciences} \bibinfo{volume}{316}
  (\bibinfo{year}{2015}) \bibinfo{pages}{308--328}.
%
\bibitem[{Halim et~al.(2015)Halim, Waqas and
  Hussain}]{HalimETAL2015:largeclust}
\bibinfo{author}{Z.~Halim}, \bibinfo{author}{M.~Waqas}, \bibinfo{author}{S.F.
  Hussain}, \bibinfo{title}{Clustering large probabilistic graphs using
  multi-population evolutionary algorithm}, \bibinfo{journal}{Information
  Sciences} \bibinfo{volume}{317} (\bibinfo{year}{2015})
  \bibinfo{pages}{78--95}.
%
\bibitem[{Hastie et~al.(2013)Hastie, Tibshirani and
  Friedman}]{TibEtAll:elementsstat}
\bibinfo{author}{T.~Hastie}, \bibinfo{author}{R.~Tibshirani},
  \bibinfo{author}{J.~Friedman}, \bibinfo{title}{The Elements of Statistical
  Learning: Data Mining, Inference, and Prediction},
  \bibinfo{publisher}{Springer}, \bibinfo{year}{2013}.
%
\bibitem[{Heip(1974)}]{Heip1974:evenness}
\bibinfo{author}{C.~Heip}, \bibinfo{title}{A new index measuring evenness},
  \bibinfo{journal}{Journal of {M}arine {B}iological {A}ssociation of the
  {U}nited {K}ingdom} \bibinfo{volume}{54} (\bibinfo{year}{1974})
  \bibinfo{pages}{555--557}.
%
\bibitem[{Jain and Law(2005)}]{JainLaw2005:dilemma}
\bibinfo{author}{A.~Jain}, \bibinfo{author}{M.~Law}, \bibinfo{title}{Data
  clustering: {A} user's dilemma}, \bibinfo{journal}{Lecture Notes in Computer
  Science} \bibinfo{volume}{3776} (\bibinfo{year}{2005})
  \bibinfo{pages}{1--10}.
%
\bibitem[{Jiang et~al.(2016)Jiang, Liu, Du and Sui}]{JiangETAL2016:kmodes}
\bibinfo{author}{F.~Jiang}, \bibinfo{author}{G.~Liu}, \bibinfo{author}{J.~Du},
  \bibinfo{author}{Y.~Sui}, \bibinfo{title}{Initialization of k-modes
  clustering using outlier detection techniques}, \bibinfo{journal}{Information
  Sciences} \bibinfo{volume}{332} (\bibinfo{year}{2016})
  \bibinfo{pages}{167--183}.
%
\bibitem[{Kobus(2012)}]{Kobus2012:inequalitydecomp}
\bibinfo{author}{M.~Kobus}, \bibinfo{title}{Attribute decomposition of
  multidimensional inequality indices}, \bibinfo{journal}{Economics Letters}
  \bibinfo{volume}{117} (\bibinfo{year}{2012}) \bibinfo{pages}{189--191}.
%
\bibitem[{Kobus and Miłoś(2012)}]{KobusMilos2012:inequalitydecomp}
\bibinfo{author}{M.~Kobus}, \bibinfo{author}{P.~Miłoś},
  \bibinfo{title}{Inequality decomposition by population subgroups for ordinal
  data}, \bibinfo{journal}{Journal of Health Economics} \bibinfo{volume}{31}
  (\bibinfo{year}{2012}) \bibinfo{pages}{15--21}.
%
\bibitem[{Kruskal(1956)}]{KruskalProof}
\bibinfo{author}{J.B. Kruskal}, \bibinfo{title}{On the shortest spanning
  subtree of a graph and the traveling salesman problem},
  \bibinfo{journal}{Proceedings of the American Mathematical Society}
  \bibinfo{volume}{7} (\bibinfo{year}{1956}) \bibinfo{pages}{48--50}.
%
\bibitem[{Kärkkäinen and Fränti(2002)}]{KarkkainenFranti2002:asets}
\bibinfo{author}{I.~Kärkkäinen}, \bibinfo{author}{P.~Fränti},
  \bibinfo{title}{Dynamic local search algorithm for the clustering problem},
  in: \bibinfo{booktitle}{Proc. 16th Intl. Conf. Pattern Recognition'02},
  volume~\bibinfo{volume}{2}, \bibinfo{publisher}{IEEE}, \bibinfo{year}{2002},
  pp. \bibinfo{pages}{240--243}.
%
\bibitem[{Legendre and Legendre(2003)}]{LegendreLegendre2003:numericalecology}
\bibinfo{author}{P.~Legendre}, \bibinfo{author}{L.~Legendre},
  \bibinfo{title}{Numerical Ecology}, \bibinfo{publisher}{Elsevier Science BV},
  \bibinfo{address}{Amsterdam}, \bibinfo{year}{2003}.
%
\bibitem[{MacQueen(1967)}]{MacQueen1967:kmeans}
\bibinfo{author}{J.B. MacQueen}, \bibinfo{title}{Some methods for
  classification and analysis of multivariate observations}, in:
  \bibinfo{booktitle}{Proc. Fifth Berkeley Symp. on Math. Statist. and Prob.},
  volume~\bibinfo{volume}{1}, \bibinfo{publisher}{University of California
  Press}, \bibinfo{address}{Berkeley}, \bibinfo{year}{1967}, pp.
  \bibinfo{pages}{281--297}.
%
\bibitem[{M\"ullner(2011)}]{Mullner2011:fastclusteralg}
\bibinfo{author}{D.~M\"ullner}, \bibinfo{title}{Modern hierarchical,
  agglomerative clustering algorithms}, \bibinfo{journal}{ArXiv:1109.2378
  [stat.ML]}  (\bibinfo{year}{2011}).
%
\bibitem[{M\"ullner(2013)}]{Mullner2013:fastcluster}
\bibinfo{author}{D.~M\"ullner}, \bibinfo{title}{{fastcluster}: Fast
  hierarchical, agglomerative clustering routines for {R} and {Python}},
  \bibinfo{journal}{Journal of Statistical Software} \bibinfo{volume}{53}
  (\bibinfo{year}{2013}) \bibinfo{pages}{1--18}.
%
\bibitem[{Murtagh(1983)}]{Murtagh:survey}
\bibinfo{author}{F.~Murtagh}, \bibinfo{title}{A survey of recent advances in
  hierarchical clustering algorithms}, \bibinfo{journal}{The Computer Journal}
  \bibinfo{volume}{26} (\bibinfo{year}{1983}) \bibinfo{pages}{354--359}.
%
\bibitem[{Murtagh and Legendre(2014)}]{MurtaghLegendre2014:ward}
\bibinfo{author}{F.~Murtagh}, \bibinfo{author}{P.~Legendre},
  \bibinfo{title}{Ward's hierarchical agglomerative clustering method: {W}hich
  algorithms implement ward's criterion?}, \bibinfo{journal}{Journal of
  Classification} \bibinfo{volume}{31} (\bibinfo{year}{2014})
  \bibinfo{pages}{274--295}.
%
\bibitem[{Olson(1995)}]{Olson1995:parallelhierclust}
\bibinfo{author}{C.F. Olson}, \bibinfo{title}{Parallel algorithms for
  hierarchical clustering}, \bibinfo{journal}{Parallel Computing}
  \bibinfo{volume}{21} (\bibinfo{year}{1995}) \bibinfo{pages}{1313--1325}.
%
\bibitem[{Pedrycz(1996)}]{Pedrycz1996:conditionalfcm}
\bibinfo{author}{W.~Pedrycz}, \bibinfo{title}{Conditional fuzzy c-means},
  \bibinfo{journal}{Pattern Recognition Letters} \bibinfo{volume}{17}
  (\bibinfo{year}{1996}) \bibinfo{pages}{625--631}.
%
\bibitem[{Pedrycz and Bargiela(2002)}]{PedryczBargiela2002:granclust}
\bibinfo{author}{W.~Pedrycz}, \bibinfo{author}{A.~Bargiela},
  \bibinfo{title}{Granular clustering: {A} granular signature of data},
  \bibinfo{journal}{IEEE Transactions on Systems, Man, and Cybernetics, Part B:
  Cybernetics} \bibinfo{volume}{32} (\bibinfo{year}{2002})
  \bibinfo{pages}{212--224}.
%
\bibitem[{Pedrycz and Waletzky(1997)}]{PedryczWaletzky1997:fuzclustsup}
\bibinfo{author}{W.~Pedrycz}, \bibinfo{author}{J.~Waletzky},
  \bibinfo{title}{Fuzzy clustering with partial supervision},
  \bibinfo{journal}{IEEE Transactions on Systems, Man, and Cybernetics, Part B:
  Cybernetics} \bibinfo{volume}{27} (\bibinfo{year}{1997})
  \bibinfo{pages}{787--795}.
%
\bibitem[{Pielou(1969)}]{Pielou1969:introecology}
\bibinfo{author}{E.~Pielou}, \bibinfo{title}{An {I}ntroduction to
  {M}athematical {E}cology}, \bibinfo{publisher}{Wiley-{I}nterscience},
  \bibinfo{address}{New {Y}ork}, \bibinfo{year}{1969}.
%
\bibitem[{Pielou(1975)}]{Pielou1975:ecodiversity}
\bibinfo{author}{E.~Pielou}, \bibinfo{title}{Ecological {D}iversity},
  \bibinfo{publisher}{Wiley}, \bibinfo{address}{New {Y}ork},
  \bibinfo{year}{1975}.
%
\bibitem[{Prim(1957)}]{Prim1957:MST}
\bibinfo{author}{R.~Prim}, \bibinfo{title}{Shortest connection networks and
  some generalizations}, \bibinfo{journal}{Bell System Technical Journal}
  \bibinfo{volume}{36} (\bibinfo{year}{1957}) \bibinfo{pages}{1389--1401}.
%
\bibitem[{{R Development Core Team}(2015)}]{Rproject:home}
\bibinfo{author}{{R Development Core Team}}, \bibinfo{title}{\textsf{R}:
  A~language and environment for statistical computing},
  \bibinfo{organization}{\textsf{R} Foundation for Statistical Computing},
  \bibinfo{address}{Vienna, Austria}, \bibinfo{year}{2015}.
  \bibinfo{note}{{http://www.R-project.org}}.
%
\bibitem[{Radavanovic et~al.(2010)Radavanovic, Nanopoulos and
  Ivanovic}]{RadavanovicETAL2010:hubs}
\bibinfo{author}{M.~Radavanovic}, \bibinfo{author}{A.~Nanopoulos},
  \bibinfo{author}{M.~Ivanovic}, \bibinfo{title}{Hubs in space: {P}opular
  nearest neighbors in high-dimensional data}, \bibinfo{journal}{Journal of
  Machine Learning Research} \bibinfo{volume}{11} (\bibinfo{year}{2010})
  \bibinfo{pages}{2487--2531}.
%
\bibitem[{Rohlf(1973)}]{Rohlf1973:mst}
\bibinfo{author}{F.~Rohlf}, \bibinfo{title}{Hierarchical clustering using the
  minimum spanning tree}, \bibinfo{journal}{The Computer Journal}
  \bibinfo{volume}{16} (\bibinfo{year}{1973}) \bibinfo{pages}{93--95}.
%
\bibitem[{Veenman et~al.(2002)Veenman, Reinders and
  Backer}]{VeenmanETAL2002:maxvar}
\bibinfo{author}{C.~Veenman}, \bibinfo{author}{M.~Reinders},
  \bibinfo{author}{E.~Backer}, \bibinfo{title}{A maximum variance cluster
  algorithm}, \bibinfo{journal}{IEEE Transactions on Pattern Analysis and
  Machine Intelligence} \bibinfo{volume}{24} (\bibinfo{year}{2002})
  \bibinfo{pages}{1273--1280}.
%
\bibitem[{Xu and {Wunsch II}(2009)}]{XuWunsch:clustering}
\bibinfo{author}{R.~Xu}, \bibinfo{author}{D.C. {Wunsch II}},
  \bibinfo{title}{Clustering}, \bibinfo{publisher}{Wiley-IEEE Press},
  \bibinfo{year}{2009}.
%
\bibitem[{Yianilos(1993)}]{Yianilos1993:vptree}
\bibinfo{author}{P.N. Yianilos}, \bibinfo{title}{Data structures and algorithms
  for nearest neighbor search in general metric spaces}, in:
  \bibinfo{booktitle}{Proceedings of the Fourth Annual ACM-SIAM Symposium on
  Discrete Algorithms}, SODA '93, \bibinfo{publisher}{Society for Industrial
  and Applied Mathematics}, \bibinfo{year}{1993}, pp.
  \bibinfo{pages}{311--321}.
%
\bibitem[{Zahn(1971)}]{Zahn1971:gestalt}
\bibinfo{author}{C.~Zahn}, \bibinfo{title}{Graph-theoretical methods for
  detecting and describing gestalt clusters}, \bibinfo{journal}{IEEE
  Transactions on Computers} \bibinfo{volume}{C-20} (\bibinfo{year}{1971})
  \bibinfo{pages}{68--86}.
%
\bibitem[{Zahra et~al.(2015)Zahra, Ghazanfar, Khalid, Azam, Naeem and
  Prugel-Bennett}]{ZahraETAL2015:kmeans}
\bibinfo{author}{S.~Zahra}, \bibinfo{author}{M.A. Ghazanfar},
  \bibinfo{author}{A.~Khalid}, \bibinfo{author}{M.A. Azam},
  \bibinfo{author}{U.~Naeem}, \bibinfo{author}{A.~Prugel-Bennett},
  \bibinfo{title}{Novel centroid selection approaches for kmeans-clustering
  based recommender systems}, \bibinfo{journal}{Information Sciences}
  \bibinfo{volume}{320} (\bibinfo{year}{2015}) \bibinfo{pages}{156--189}.
%
\bibitem[{Zhang et~al.(1996)Zhang, Ramakrishnan and
  Livny}]{ZhangETAL1996:BIRCH}
\bibinfo{author}{T.~Zhang}, \bibinfo{author}{R.~Ramakrishnan},
  \bibinfo{author}{M.~Livny}, \bibinfo{title}{{BIRCH}: {A}n efficient data
  clustering method for very large databases}, in: \bibinfo{booktitle}{Proc.
  ACM SIGMOD'96 Intl. Conf. Management of Data}, \bibinfo{publisher}{ACM},
  \bibinfo{year}{1996}, pp. \bibinfo{pages}{103--114}.

\end{thebibliography}

\end{document}